\documentclass{article}

\usepackage{iclr2027_conference,times}
\iclrfinalcopy

\usepackage[utf8]{inputenc}
\usepackage[T1]{fontenc}
\usepackage{courier}

\usepackage{url}
\usepackage{booktabs}
\usepackage{amsfonts}
\usepackage{amsmath}
\usepackage{amssymb}
\usepackage{amsthm}
\usepackage{nicefrac}
\usepackage{microtype}
\usepackage{xcolor}
\usepackage{graphicx}
\usepackage{float}
\usepackage{multirow}
\usepackage{subcaption}
\newfloat{algorithm}{t}{loa}
\floatname{algorithm}{Algorithm}
\usepackage{hyperref}

\newtheorem{proposition}{Proposition}

\newtheorem{remark}[proposition]{Remark}

\DeclareMathOperator{\softmax}{softmax}
\DeclareMathOperator{\Var}{Var}
\newcommand{\R}{\mathbb{R}}
\newcommand{\E}{\mathbb{E}}
\newcommand{\M}{\mathcal{M}}
\newcommand{\inner}[2]{\langle #1, #2 \rangle}
\newcommand{\norm}[1]{\|#1\|}

\title{Model Compression with Exact Budget\\Constraints via Riemannian Manifolds}

\author{Michael Helcig\thanks{Corresponding author: \texttt{mhelcig@ethz.ch}} \\
  ETH Z\"urich \\
  \And
  Dan Alistarh \\
  ISTA \\
}

\begin{document}

\maketitle
\lhead{Preprint}

\begin{abstract}
Assigning one of $K$ options to each of $N$ groups under a total cost budget is a recurring problem in efficient AI, for instance in mixed-precision quantization, non-uniform pruning, and expert selection.
The objective (model loss) depends on all assignments jointly and does not decompose across groups, which means that combinatorial solvers can only optimize proxy objectives. 
Methods such as evolutionary search evaluate the actual loss but lack gradients, while penalty-based methods  enforce the budget only approximately and can require heavy hyperparameter tuning. We show that under softmax relaxation, the budget constraint defines a smooth Riemannian manifold in logit space with unusually clean geometry. Specifically, the normal vector is available in closed form, shifting logits along the cost vector changes expected cost monotonically, and the vector transport reduces to a single inner product.
Building on this, we propose Riemannian Constrained Optimization (RCO), which wraps tangent projection, binary-search retraction, and momentum transport around a standard Adam step.
Combined with Gumbel straight-through estimation and budget-constrained dynamic programming for discrete feasibility, RCO provides first-order optimization of the actual loss with the expected budget enforced exactly at every iterate, and no constraint-related hyperparameters. The same construction handles several simultaneous budgets, whenever an assignment must respect more than one resource limit, with no added coefficients.
RCO exceeds or matches the performance of state-of-the-art methods in both synthetic problems and realistic LLM compression settings, often at considerably lower wall-clock cost.

\end{abstract}

\section{Introduction}
\label{sec:intro}

Constrained optimization over discrete choices is a pervasive problem in efficient machine learning. A common instance is \textbf{budget-constrained discrete assignment}: choose one of $K$ options for each of $N$ target groups (e.g., a compression level for each layer of a neural network), minimizing an objective that depends on all assignments jointly (the model accuracy loss), subject to a total cost constraint (the total model size).
The joint dependence between groups (e.g., interactions between layers) makes the objective non-decomposable, not a sum of per-group terms.
Dynamic programming (DP), the standard tool for budget-constrained assignment, relies precisely on this additive structure, and can therefore only optimize proxy objectives.
Concretely, variants of this problem appear in model compression, where each layer or block takes a different compression degree, e.g., mixed-precision quantization~\citep{dong2019hawq, dong2020hawqv2, yao2021hawqv3}, non-uniform pruning~\citep{frantar2022spdy, yin2024owl}, MoE expert pruning~\citep{lu2024eep, lasby2025reap, liu2026evoesap}, and even MoE routing under load-balancing constraints~\citep{Fedus2021SwitchTransformers, zhou2022expert}.

\begin{figure}[h!]
  \centering
  \includegraphics[width=\linewidth]{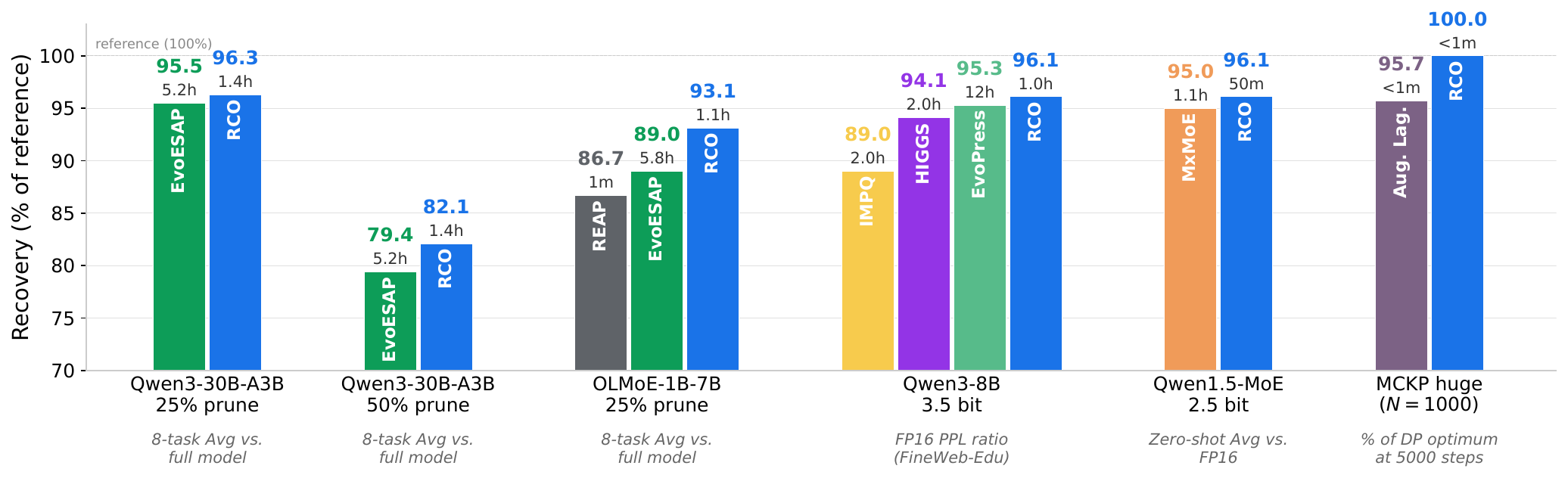}
  \caption{Recovery vs.\ baselines: RCO matches or exceeds each at lower wall-clock. Bars show recovery (\%, bold) and wall-clock time (small).}
  \label{fig:summary}
\end{figure}

Sensitivity methods score groups independently, giving each layer a ``sensitivity'' and allocating via dynamic programming or integer linear programming~\citep{dong2020hawqv2, yao2021hawqv3, frantar2022spdy, li2023llmmq, malinovskii2024higgs, duanmu2025mxmoe}, but this fails at high compression rates~\citep{sieberling2024evopress}.
Directly evaluating the actual model loss avoids this, but existing methods require hundreds of forward passes to converge~\citep{sieberling2024evopress, liu2026evoesap}.
Gradient-based relaxation could be more efficient, but a quadratic penalty $\lambda\,(C(\boldsymbol{\alpha}) - B)^2$ on the budget deviation never satisfies the budget exactly, and $\lambda$ needs careful tuning: too small violates the budget, too large lets the constraint gradient dominate, and on hard instances the iterates oscillate (Section~\ref{sec:mckp}); several simultaneous budgets then need one coupled coefficient each.
Projection-free methods such as Frank-Wolfe maintain feasibility by construction~\citep{nayman2021hardcore} but do not transport adaptive optimizer state.

We parameterize each group's assignment by logits $\boldsymbol{\alpha}_i \in \R^K$ with $\mathbf{p}_i = \softmax(\boldsymbol{\alpha}_i)$; with option costs $\mathbf{c} = (c_1, \ldots, c_K) \in \R^K$ and group weights $w_i$, the total expected cost $C(\boldsymbol{\alpha}) = \sum_i w_i \inner{\mathbf{p}_i}{\mathbf{c}}$ is a smooth function of $\boldsymbol{\alpha} \in \R^{NK}$. Given a calibration loss $L$ and budget $B$, we solve $\min_{\boldsymbol{\alpha}} L(\boldsymbol{\alpha})$ subject to $C(\boldsymbol{\alpha}) = B$, with logits $\boldsymbol{\alpha}$ the variables and $\mathbf{c}$, $w_i$, $B$ fixed.

The feasible set $\M = \{C(\boldsymbol{\alpha}) = B\}$ is well-behaved: $\nabla C$ is everywhere nonzero for distinct option costs, so $C - B$ has constant rank one and $\M$ is a smooth $(NK{-}1)$-dimensional embedded Riemannian submanifold of $\R^{NK}$ under the inherited metric~\citep[Defs.~3.10, 3.55]{boumal2023introduction}. We call this the \emph{budget manifold} and optimize directly on it.

The efficiently-accessible geometry of $\M$ rests on the softmax parameterization:
\begin{itemize}
  \item \textbf{Closed-form normal.} $(\nabla C)_{ik} = w_i\, p_{ik}(c_k - \E_{p_i}[c])$ (Eq.~\ref{eq:normal}); projecting onto the tangent plane costs a single inner product.
  \item \textbf{Monotonic retraction.} Shifting logits along $\mathbf{c}$ changes $C$ with derivative $\sum_i w_i\,\Var_{p_i}[c] > 0$ (Proposition~\ref{prop:retraction}); binary search returns any iterate to $\M$ exponentially fast.
  \item \textbf{Cheap transport.} Since $\M$ has codimension one, vector transport is another inner product.
\end{itemize}
Wrapped around a standard Adam step, these give first-order optimization under an exact budget at every iterate. Here \emph{exact} is the continuous expected budget; decoded discrete models are feasible up to a small cost-grid rounding (e.g. 2.251 vs.\ 2.25 bits; Section~\ref{sec:algorithm}).

\paragraph{Contributions.}
\begin{itemize}\setlength{\itemsep}{0pt}
  \item \textbf{The budget manifold} (Section~\ref{sec:method}). The expected-cost level set is a classical m-flat submanifold whose geometry, under the Euclidean (rather than Fisher) ambient metric, is closed-form and cheap (Proposition~\ref{prop:retraction}).
  \item \textbf{Riemannian Constrained Optimization (RCO)} (Section~\ref{sec:application}, Appendix~\ref{app:bias}). We wrap projection, retraction, and transport around Adam with no constraint-related hyperparameters (Gumbel-STE and budget-constrained DP handle forward-pass feasibility); tangent projection removes the constraint-normal component of any STE gradient bias (Proposition~\ref{prop:bias}), so the budget holds exactly regardless of gradient quality.
  \item \textbf{Multiple constraints} (Section~\ref{sec:multi}, Appendix~\ref{app:multi}). RCO enforces $q$ simultaneous constraints with no coefficients (one $q{\times}q$ solve and a $q$-dimensional Newton retraction per step; independent binary searches when disjoint), e.g.\ a per-device memory limit, equal experts per layer for GGUF export, or joint sparsity and bitwidth; Appendix~\ref{app:multi_exp} evaluates the last two.
  \item \textbf{Empirical validation} (Section~\ref{sec:experiments}). On synthetic knapsack, RCO reaches 1\% of the DP optimum (Lagrangian: 37.6\%); on LLM compression it matches or exceeds evolutionary search at 3--16$\times$ lower wall-clock.
\end{itemize}

\section{The Budget Manifold}
\label{sec:method}

We briefly review the relevant concepts from Riemannian geometry; for a textbook treatment, see \citet{boumal2023introduction}.
A \emph{manifold} is a smooth surface that may be curved globally but locally resembles flat Euclidean space.
A \emph{Riemannian manifold} additionally carries an inner product on each tangent space, providing notions of length, angle, and gradient on the surface.
At each point, the \emph{tangent plane} is the subspace of directions tangent to the surface, and the \emph{normal vector} points perpendicular to it.
Optimizing on a manifold requires three operations: \emph{tangent projection} (restricting gradients to the tangent plane), \emph{retraction} (mapping an iterate that has drifted off the surface back onto it), and \emph{vector transport} (moving vectors such as optimizer momentum from one tangent plane to another as the iterate moves along the surface).

Consider $N$ groups, each to be assigned one of $K$ discrete options.
Group $i$ carries weight $w_i > 0$ (e.g., the number of parameters in a network layer), and option $k$ has cost $c_k > 0$ (e.g., bits per parameter).
We parameterize the assignment distribution for group $i$ by logits $\boldsymbol{\alpha}_i \in \R^K$ with $\mathbf{p}_i = \softmax(\boldsymbol{\alpha}_i)$, and define the total expected cost as
\begin{equation}
  \label{eq:cost}
  C(\boldsymbol{\alpha})
  \;=\; \sum_{i=1}^{N} w_i \sum_{k=1}^{K} p_{ik}\, c_k
  \;=\; \sum_{i=1}^{N} w_i \inner{\mathbf{p}_i}{\mathbf{c}}.
\end{equation}
Given a budget $B$ with $\min_k c_k < B / \sum_i w_i < \max_k c_k$, the feasible set is $\M = \{ \boldsymbol{\alpha} \in \R^{NK} : C(\boldsymbol{\alpha}) = B \}$, and the goal is to minimize an objective $L(\boldsymbol{\alpha})$ on~$\M$.

\begin{proposition}[Manifold structure]
  \label{prop:manifold}
  $C$ is smooth.
  If the costs $c_1, \ldots, c_K$ are not all equal, then $\nabla C(\boldsymbol{\alpha}) \neq \mathbf{0}$ for all $\boldsymbol{\alpha}$, and $\M$ is a smooth $(NK{-}1)$-dimensional embedded submanifold of $\R^{NK}$, Riemannian under the metric inherited from $\R^{NK}$.
\end{proposition}

The proof is in Appendix~\ref{app:manifold_proofs}. The gradient is nonzero because softmax produces strictly positive probabilities, so every group has positive cost variance.
Concretely, $\M$ is a smooth surface in logit space: at every point it has a well-defined tangent plane (budget-preserving directions) and a normal direction (the budget-changing direction), which are the two ingredients needed for constrained optimization on the surface.

\begin{proposition}[Normal vector]
  \label{prop:normal}
  The gradient of $C$ with respect to $\boldsymbol{\alpha}$ is
  \begin{equation}
    \label{eq:normal}
    \bigl(\nabla C(\boldsymbol{\alpha})\bigr)_{ik}
    \;=\; w_i\, p_{ik}\bigl(c_k - \E_{p_i}[c]\bigr),
  \end{equation}
  where $\E_{p_i}[c] = \sum_k p_{ik}\, c_k$ (Appendix~\ref{app:manifold_proofs}).
  This vector is normal to $\M$ at every point.
\end{proposition}

\begin{figure}[h!]
  \centering
  \includegraphics[width=0.92\linewidth]{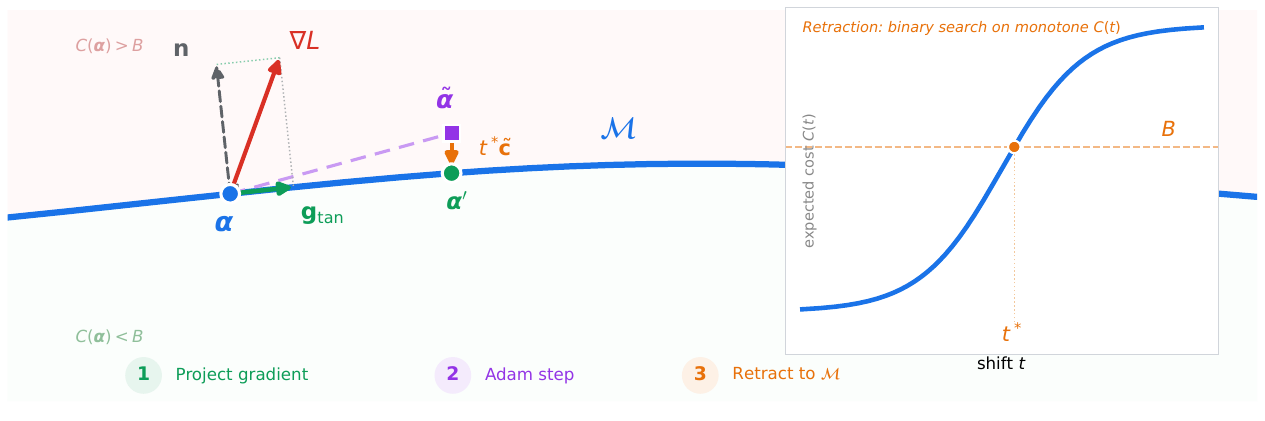}
  \caption{One optimization step on $\M$: project gradient, Adam step, retract via binary search.}
\label{fig:method}
\end{figure}

We optimize $L$ on $\M$ by wrapping three operations around a standard Adam step (Figure~\ref{fig:method}).
First, we project the loss gradient $\mathbf{g} = \nabla L(\boldsymbol{\alpha})$ onto the tangent space of $\M$ by subtracting its component along the normal $\mathbf{n} = \nabla C(\boldsymbol{\alpha})$: $\mathbf{g}_{\mathrm{tan}} = \mathbf{g} - (\inner{\mathbf{g}}{\mathbf{n}} / \norm{\mathbf{n}}^2)\,\mathbf{n}$.
This removes the budget-changing direction, so the optimizer sees only budget-preserving gradients.
We project before the Adam update rather than after so that the first moment $\mathbf{m}$ accumulates a tangent direction across steps; without the pre-projection, momentum would drift toward the normal and be wasted by retraction.
An Adam step along $\mathbf{g}_{\mathrm{tan}}$ may still drift off $\M$ due to manifold curvature and Adam's per-coordinate scaling.
To return, we shift all logits along the cost vector, setting $\alpha'_{ik} = \alpha_{ik} + t\, c_k$, and binary-search for $t$ such that $C(\boldsymbol{\alpha}') = B$.
This works because of the following structural property specific to softmax:

\begin{proposition}[Monotonic retraction]
  \label{prop:retraction}
  Let $\tilde{c}_{ik} = c_k$ (the cost vector broadcast across groups) and $p_i(t) = \softmax(\boldsymbol{\alpha}_i + t\mathbf{c})$.
  Then $C(\boldsymbol{\alpha} + t\tilde{\mathbf{c}})$ is strictly increasing in $t$ whenever the costs are not all equal (Appendix~\ref{app:retraction_proof}): 
  \begin{equation}
    \label{eq:retraction_deriv}
    \frac{d}{dt}\, C(\boldsymbol{\alpha} + t\tilde{\mathbf{c}})
    \;=\; \sum_{i=1}^{N} w_i\, \Var_{p_i(t)}[c]
    \;>\; 0. 
  \end{equation}
  
\end{proposition}

Shifting logits by $t\mathbf{c}$ adds $t(c_k - c_{k'})$ to each logit difference, so higher-cost options gain probability monotonically and expected cost increases; binary search therefore returns any iterate to $\M$ exponentially fast.
The closed-form derivative~\eqref{eq:retraction_deriv} also enables Newton retraction in 2--4 iterations, which the multi-constraint extension (Appendix~\ref{app:multi}) builds on; we use binary search for the single-constraint case.
This cheapness is specific to working in logit space with the Euclidean ambient metric: standard Fisher-metric treatments of the simplex~\citep{amari2000methods, lebanon2005riemannian} replace $\nabla C$ with $F^{-1}\nabla C$ in projection and transport and follow Fisher geodesics for retraction. The Euclidean choice replaces every $F^{-1}$ with the identity, and retracting along $\tilde{\mathbf{c}}$ exploits cumulant-generating-function strict convexity of exponential families~\citep[Prop.~3.1]{wainwright2008graphical} to reduce retraction to a scalar monotone binary search. Other manifolds lack this structure: Stiefel uses QR or polar decomposition, fixed-rank uses truncated SVD, doubly stochastic uses Sinkhorn~\citep{douik2019manifold}, and generic smooth constraints need nonlinear root-finding~\citep[\S7.7]{boumal2023introduction}.

After retraction to $\boldsymbol{\alpha}' \in \M$, the tangent plane has rotated, so we project Adam's first moment onto the new tangent space (vector transport by projection; Appendix~\ref{app:transport}): $\mathbf{m} \leftarrow \mathbf{m} - (\inner{\mathbf{m}}{\mathbf{n}'} / \norm{\mathbf{n}'}^2)\, \mathbf{n}'$.
Only the first moment is transported (cross-step momentum that must remain tangent); the second moment is per-coordinate magnitude with no tangent structure. Adam's per-coordinate scaling does rotate the update off the tangent plane (Remark~\ref{rem:adam_bias}), corrected by retraction at every step.
These three operations wrap around any first-order optimizer at negligible cost: one inner product each for projection and transport, and roughly $\lceil \log_2(R/\varepsilon) \rceil$ softmax evaluations for retraction, where $R$ is the bracket width on the shift parameter $t$ and $\varepsilon$ the cost-residual tolerance; for $R = 100$ and $\varepsilon = 10^{-8}$ this gives ${\sim}\,34$ iterations in theory and ${\le}\,45$ across all MCKP scenarios in practice (Figure~\ref{fig:retraction_iters}, Appendix~\ref{app:retraction_proof}).
The retraction, projected gradient, and transport satisfy the standard manifold-optimization axioms (Appendices~\ref{app:retraction_axioms}--\ref{app:transport}), so with exact gradients Riemannian gradient descent drives the Riemannian gradient norm to zero under standard regularity conditions (cost bounded below, a Lipschitz-type bound on the pullback through the retraction)~\citep[Chapter~4]{boumal2023introduction}; Adam is a practical heuristic on top, validated empirically.

\paragraph{Multiple constraints.}
\label{sec:multi}
The construction extends to $q$ simultaneous constraints $C_j(\boldsymbol{\alpha}) = b_j$ with normals $\mathbf{N} = [\nabla C_1 \cdots \nabla C_q]$ (Proposition~\ref{prop:normal}): projection and transport subtract $\mathbf{N}(\mathbf{N}^\top\mathbf{N})^{-1}\mathbf{N}^\top(\cdot)$, one $q{\times}q$ solve per step, and retraction solves $C_j = b_j$ for per-group shifts by Newton's method with a closed-form Jacobian; for budgets on disjoint groups the system is diagonal and each step is $q$ independent binary searches (Appendix~\ref{app:multi}). No coefficient enters for any $q$, whereas a penalty needs one per budget. Such budgets arise wherever several resource limits apply at once, on disjoint parts of the model (e.g.\ a per-device memory limit, or the equal expert count per layer for GGUF export~\citep{gerganov2023llamacpp}) or over the same decisions (e.g.\ joint sparsity and average bitwidth on the same experts); Appendix~\ref{app:multi_exp} evaluates one instance of each kind, the shared-group one against a penalty baseline, and Appendix~\ref{app:penalty_ablation} repeats the comparison for single budgets. Inequality constraints use a slack variable $s$ with $C(\boldsymbol{\alpha}) + s^2 = B$ (Appendix~\ref{app:slack}).

\section{The Riemannian Constrained Optimization (RCO) Algorithm}
\label{sec:application}

The budget manifold (Section~\ref{sec:method}) provides three geometric operations for constrained optimization: tangent projection, monotonic retraction, and momentum transport.
To produce discrete assignments from continuous logits, we combine these operations with Gumbel STE and DP into a complete algorithm we call \emph{Riemannian Constrained Optimization} (RCO).

The Gumbel-STE gradient is biased, so the exact-gradient guarantee of Section~\ref{sec:method} does not extend to RCO. The tangent projection removes the constraint-normal component of any gradient bias (Proposition~\ref{prop:bias}), so the expected budget is enforced exactly regardless of gradient quality; the tangential component remains, so objective descent and the quality of the final assignment are established empirically rather than proved. Convergence guarantees for straight-through training exist only in restricted settings, such as two-layer networks with binary activations~\citep{yin2019understanding}.

\subsection{Algorithm}
\label{sec:algorithm}

RCO enforces the budget in two complementary spaces: budget-constrained DP gives discrete feasibility in the forward pass, the manifold projection gives continuous feasibility in the backward pass.
We retain the setting from Section~\ref{sec:method}: $N$ groups, each with $K$ options of cost $c_k$ and group weights $w_i$, and a total cost budget $B$.
Concretely, each forward pass draws Gumbel noise~\citep{jang2017categorical,maddison2017concrete} $G_{ik}$, forms perturbed logits $\hat{\alpha}_{ik} = (\alpha_{ik} + G_{ik})/\tau$ at temperature $\tau$, and solves the constrained assignment
\begin{equation}
  \label{eq:dp}
  \mathbf{z}^* = \arg\max_{\mathbf{z} \in \{0,1\}^{NK}}
  \sum_{i,k} z_{ik}\, \hat{\alpha}_{ik}
  \;\;\text{s.t.}\;\;
  \sum_{i} w_i \sum_{k} z_{ik}\, c_k \leq B,
  \;\; \sum_{k} z_{ik} = 1 \;\;\forall\, i,
\end{equation}
which is a multiple-choice knapsack problem solved exactly by DP in $O(NKB')$ time, where $B'$ is the budget expressed in integer cost units (e.g., total bits for quantization, total kept experts for pruning) so the DP table has $B' + 1$ entries per group.
Per-group costs are rounded to this integer grid, the source of the small decoded-budget deviations discussed in Section~\ref{sec:intro}.
For two options with costs $0$ and $1$ and a budget fixing their count, as in expert pruning, Eq.~\eqref{eq:dp} reduces to selecting the $B'$ largest perturbed logit differences, an $O(N \log N)$ top-$k$, applied per layer when each layer has its own budget (Appendix~\ref{app:solver_cost}).
In the backward pass, the STE~\citep{bengio2013estimating} replaces the non-differentiable $\arg\max$ with soft probabilities: the forward value $z^*_{ik}$ is kept, but gradients flow through $\hat{p}_{ik} = \softmax(\hat{\boldsymbol{\alpha}}_i)_k$, the softmax of the same perturbed logits that produced $\mathbf{z}^*$. This ensures the surrogate concentrates on the sampled mode, so the STE bias vanishes as $\tau \to 0$ and independent Gumbel samples yield independent Jacobians; the unperturbed $\softmax(\boldsymbol{\alpha}_i)$ would decouple the surrogate from the sampled assignment and suppress both effects.
The tangent projection (Section~\ref{sec:method}) removes the normal component from this gradient before it reaches the optimizer, eliminating any constraint-normal STE bias (Proposition~\ref{prop:bias}) and preventing it from accumulating in Adam's momentum.
After optimization, the final assignment is extracted by solving~\eqref{eq:dp} with $\hat{\alpha}_{ik} = \log p_{ik}$ (no noise, no temperature).
Algorithm~\ref{alg:rco} gives the complete procedure.

\begin{algorithm}[t]
\caption{Riemannian Constrained Optimization (RCO)}
\label{alg:rco}
\hrule \vspace{4pt}
\small
\textbf{Input:} Logits $\boldsymbol{\alpha}_0 \in \M$, loss $L$,
costs $\mathbf{c}\in \R^K$, weights $\mathbf{w}\in \R^N$,
budget $B$, steps $T$, Gumbel samples $g$, temperature schedule $\{\tau_t\}$
\vspace{2pt}
1: Initialize Adam state: $\mathbf{m}=\mathbf{0}$, $\mathbf{v}=\mathbf{0}$ \\
2: \textbf{for} $t=1,\ldots,T$ \textbf{do} \\
\hspace{1em}\textit{// Forward pass: Gumbel-STE + DP, averaged over $g$ samples} \\
3: \quad $\mathbf{g}\leftarrow \mathbf{0}$ \\
4: \quad \textbf{for} $j=1,\ldots,g$ \textbf{do} \\
5: \qquad Sample $G^{(j)}_{ik}\sim \mathrm{Gumbel}(0,1)$ ;
set $\hat{\alpha}^{(j)}_{ik}=(\alpha_{ik}+G^{(j)}_{ik})/\tau_t$ \\
6: \qquad $\mathbf{z}^{(j)} \leftarrow
\text{DP-solve}(\hat{\boldsymbol{\alpha}}^{(j)}, \mathbf{c}, \mathbf{w}, B)$
\hfill \textit{\small Eq.~\eqref{eq:dp}} \\
7: \qquad $\hat{\mathbf{p}}^{(j)} \leftarrow \softmax(\hat{\boldsymbol{\alpha}}^{(j)})$ ;
$\mathbf{g}\leftarrow \mathbf{g} + \tfrac{1}{g}\,\nabla_{\boldsymbol{\alpha}} L\bigl(\mathrm{sg}(\mathbf{z}^{(j)} - \hat{\mathbf{p}}^{(j)}) + \hat{\mathbf{p}}^{(j)}\bigr)$
\hfill \textit{\small STE at $\hat{\boldsymbol{\alpha}}^{(j)}$; contributes $1/\tau_t$} \\
8: \quad \textbf{end for} \\
\hspace{1em}\textit{// Tangent projection} \\
9: \quad $\mathbf{n} \leftarrow \nabla C(\boldsymbol{\alpha})$
\hfill \textit{\small Eq.~\eqref{eq:normal}} \\
10: \quad $\mathbf{g}\leftarrow \mathbf{g}-
(\inner{\mathbf{g}}{\mathbf{n}}/\norm{\mathbf{n}}^2)\,\mathbf{n}$ \\
\hspace{1em}\textit{// Optimizer step + retraction} \\
11: \quad $\boldsymbol{\alpha} \leftarrow
\textsc{Adam}(\boldsymbol{\alpha},\mathbf{g},\mathbf{m},\mathbf{v})$ \\
12: \quad Binary search for $t^*$ s.t.
$C(\boldsymbol{\alpha}+t^*\tilde{\mathbf{c}})=B$ ;
set $\alpha_{ik}\leftarrow \alpha_{ik}+t^*c_k$
\hfill \textit{\small Prop.~\ref{prop:retraction}} \\
\hspace{1em}\textit{// Momentum transport} \\
13: \quad $\mathbf{n}'\leftarrow \nabla C(\boldsymbol{\alpha})$ \\
14: \quad $\mathbf{m}\leftarrow \mathbf{m}-
(\inner{\mathbf{m}}{\mathbf{n}'}/\norm{\mathbf{n}'}^2)\,\mathbf{n}'$ \\
15: \textbf{end for} \\
16: \textbf{return}
$\text{DP-solve}(\log \mathbf{p}, \mathbf{c}, \mathbf{w}, B)$
\hfill \textit{\small Final discrete assignment}
\vspace{2pt}\hrule
\end{algorithm}

\paragraph{Initialization.}
\label{par:init}
The algorithm assumes $\boldsymbol{\alpha}_0 \in \M$.
If the initial logits do not satisfy $C(\boldsymbol{\alpha}_0) = B$ (e.g., when initialized from heuristic scores), one application of the binary-search retraction projects them onto $\M$ before the first step.

\paragraph{Temperature and variance reduction.}
The temperature $\tau$ in Algorithm~\ref{alg:rco} (line~5) controls exploration versus exploitation.
We anneal exponentially $\tau_t = \max\bigl(\tau_{\min},\; \tau_0 \cdot (\tau_{\min}/\tau_0)^{t/T}\bigr)$, from $\tau_0 = 1.0$ to $\tau_{\min} = 0.01$.
High temperature produces diverse DP samples and diffuse soft distributions; low temperature concentrates probability on the emerging optimum, tightening the STE approximation.
The inner loop (lines~4--8) averages gradients over $g$ independent Gumbel samples per step for variance reduction, where $g \geq 1$ is a hyperparameter (Appendix~\ref{app:quant_ablations} ablates this choice).

Setup for applying RCO to LLM compression (KL objective, weight assembly via GPTQ residuals, STE backward pass) is detailed in Appendix~\ref{app:llm_setup}.

\section{Experiments}
\label{sec:experiments}

Wall-clock comparisons against EvoPress (quantization) use matched hardware (RTX A6000); EvoESAP expert-pruning times are as reported by \citet{liu2026evoesap} on their hardware, and the appendix scaling table runs RCO on a weaker GPU (conservative).

\subsection{Controlled validation: multiple-choice knapsack}
\label{sec:mckp}

We first isolate the manifold constraint handling from the stochastic gradient estimation.
The multiple-choice knapsack problem (MCKP) admits closed-form gradients and an exact DP solution, so we can compare four constraint handling methods (manifold equality, manifold with slack variable, Lagrangian, augmented Lagrangian) on the same gradient and optimizer, varying only how they enforce the budget.
We generate 13 scenarios (3 random instances per scenario, 5000 steps each) spanning correlated costs, tight budgets, adversarial instances, under-budget optima, and a large-scale instance ($N{=}1000$, $K{=}32$).
Figure~\ref{fig:mckp} shows convergence on the largest scenario; Table~\ref{tab:mckp_full} (Appendix~\ref{app:mckp_full}) reports gap and violation across all scenarios.

\begin{figure}[!ht]
  \centering
  \includegraphics[width=\linewidth]{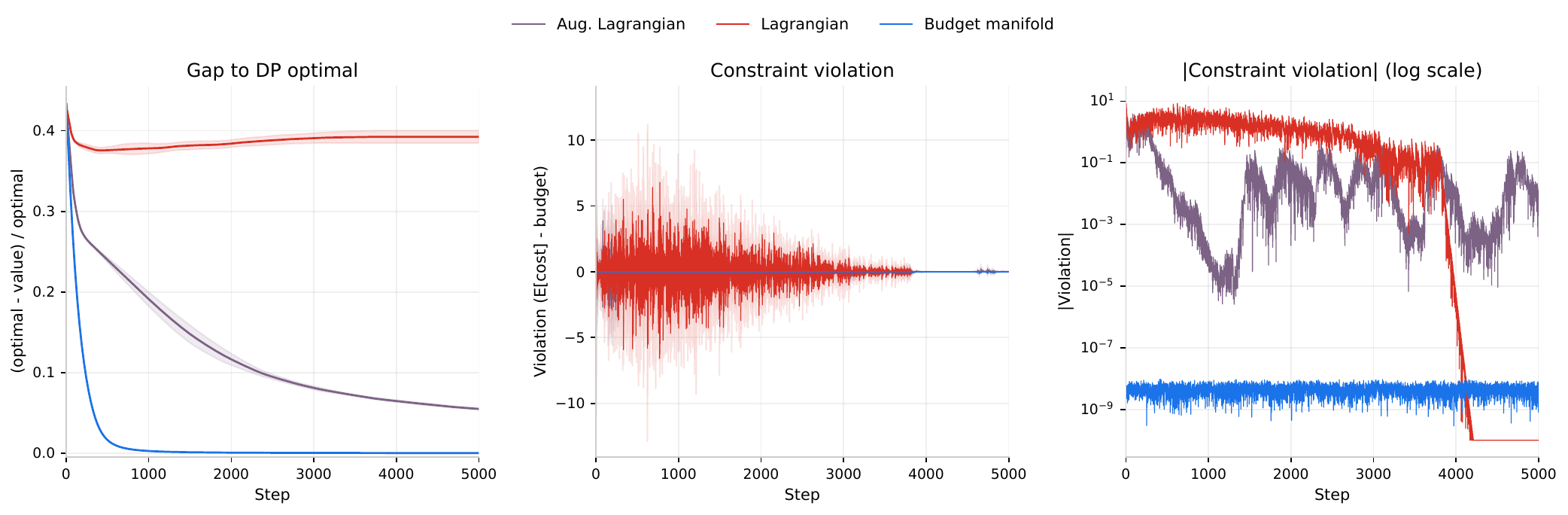}
  \caption{MCKP, \textit{huge} scenario ($N{=}1000$, $K{=}32$, three random instances).
  Left: gap to DP optimum. Center: raw constraint violation.
  Right: $|$violation$|$ on log scale; manifold ($\sim\!10^{-9}$) vs.\ Lagrangian ($\sim\!10^{-1}$).
  Additional scenarios in Appendix~\ref{app:mckp_full}.}
  \label{fig:mckp}
\end{figure}

Three patterns emerge.
First, manifold projection satisfies the budget exactly ($|C(\boldsymbol{\alpha}) - B| < 10^{-8}$) in every scenario, while Lagrangian methods maintain violations of order $10^{-1}$.
Second, at fixed compute the manifold's advantage at scale is dramatic.
On \textit{huge}, the manifold reaches within 1\% of the DP optimum in 594 steps; vanilla Lagrangian plateaus at a 37.6\% gap from step ${\sim}500$ onward and never improves; augmented Lagrangian narrows the gap monotonically but does not reach 1\% within 5000 steps (Table~\ref{tab:mckp_compute_full}, Appendix~\ref{app:mckp_full}).
Third, when the true optimum costs less than the budget (\textit{cheap optimal}, \textit{mixed slack}), equality-constrained methods waste budget on inferior options; the slack variable (Section~\ref{sec:method}, Appendix~\ref{app:slack}) recovers the DP optimum.

\subsection{MoE expert pruning}
\label{sec:llm}

We apply the manifold algorithm to MoE expert pruning, where each expert is a group with two options (keep or prune) at binary costs and the budget fixes the total number of pruned experts.
EvoESAP~\citep{liu2026evoesap} searches over per-layer counts while holding within-layer pruning order fixed by an importance criterion (we compare against its REAP variant); RCO instead optimizes per-expert logits jointly.

\paragraph{Setup.}
We evaluate on OLMoE-1B-7B (64 experts/layer), Qwen3-30B-A3B (128 experts/layer), and Qwen3-Coder-Next (512 experts/layer).
Baselines are REAP~\citep{lasby2025reap} (uniform allocation) and EvoESAP~\citep{liu2026evoesap} (evolutionary search over layer budgets), both with REAP ordering within each layer.
RCO uses expert-frequency scores only to initialize its per-expert logits; the search is free to redistribute budget across layers.
Pruned models are evaluated as-is, with no post-pruning finetuning.
We report Avg, the unweighted mean of eight standard benchmarks.
EvoESAP numbers are from \citet{liu2026evoesap}, except at Qwen3-30B-A3B 25\%, where we reproduce their search and evaluate its allocation under our protocol (Appendix~\ref{app:evoesap_details}).

\paragraph{Qwen3-30B-A3B.}
On Qwen3-30B-A3B (Figure~\ref{fig:summary}; full breakdown in Appendix~\ref{app:expert_pruning_full}, Table~\ref{tab:qwen3_full}), RCO recovers 96.3\% of the uncompressed average at 25\% pruning (71.0 Avg) in ${\sim}85$\,min, against 95.5\% reported and 95.0\% reproduced for EvoESAP (70.4 and 70.0 Avg, 5.2\,h search); uniform REAP allocation reaches 70.6 Avg, so the search does not beat uniform here, and RCO exceeds both. At 50\% pruning RCO keeps a $+2.0$ point Avg advantage over EvoESAP.

\paragraph{OLMoE-1B-7B.}
On OLMoE-1B-7B at 25\% expert sparsity, RCO at 50 steps ($\sim$10\,min) already exceeds EvoESAP (0.597 vs.\ 0.581, searched for 5.8\,h); at 300 steps it reaches 0.608 ($+$2.7 points, $5{\times}$ faster). Full per-benchmark iteration sweep in Appendix~\ref{app:expert_pruning_full} (Table~\ref{tab:olmoe_iter_full}).

\paragraph{Qwen3-Coder-Next.}
Table~\ref{tab:coder_next_main} shows expert pruning on Qwen3-Coder-Next, an 80B3A MoE model with 512 experts/layer, across calibration domain and budget allocation. Nonuniform allocation is crucial at high sparsity: at 50\%, it recovers 97\% of HumanEval versus 55\% for uniform, and at 25\% both nonuniform variants match the full model. Calibration domain induces a trade-off: coding calibration preserves code generation but hurts general knowledge, while general calibration does the opposite (Appendix~\ref{app:coder_next_full}).

\begin{table}[!ht]
  \caption{Qwen3-Coder-Next expert pruning (512~experts/layer).
  Uniform: the same expert count in every layer, imposed as one budget per layer (48 budgets, as required for GGUF export; Appendix~\ref{app:multi_disjoint}); nonuniform: one global budget that RCO distributes across layers.
  HE: HumanEval, MB: MBPP (pass@1); Avg: mean of eight general benchmarks.
  Bold: best per sparsity level. All values in \%.}
  \label{tab:coder_next_main}
  \centering
  \footnotesize
  \setlength{\tabcolsep}{2.5pt}
  \begin{tabular}{@{}l ll cc cccccccc c@{}}
    \toprule
    & & & \multicolumn{2}{c}{Coding} & \multicolumn{9}{c}{General} \\
    \cmidrule(lr){4-5} \cmidrule(lr){6-14}
    Sparsity & Cal.\ & Alloc.\ & HE & MB & ARC-C & ARC-E & BoolQ & HSwag & MMLU & OBQA & RTE & Wino & Avg \\
    \midrule
    0\% & \multicolumn{2}{l}{Full model} & 74.4 & 76.4 & 60.6 & 82.1 & 88.5 & 77.5 & 76.7 & 43.0 & 76.5 & 66.6 & 71.4 \\
    \midrule
    \multirow{4}{*}{25\%} & coding & uniform & 68.3 & \textbf{68.8} & 50.1 & 72.2 & 86.4 & 69.0 & 71.0 & 38.0 & 72.9 & 65.5 & 65.6 \\
    & coding & nonunif.\ & \textbf{74.4} & 67.8 & 46.2 & 66.2 & 85.1 & 66.5 & 68.0 & 36.2 & \textbf{77.6} & 64.2 & 63.8 \\
    & general & uniform & 4.3 & 4.6 & 60.0 & 80.7 & 87.6 & \textbf{78.5} & 70.4 & \textbf{45.2} & 75.1 & 67.7 & 70.7 \\
    & general & nonunif.\ & 6.1 & 5.8 & \textbf{61.8} & \textbf{82.2} & \textbf{88.2} & 77.6 & \textbf{71.2} & 44.2 & 76.2 & \textbf{69.9} & \textbf{71.4} \\
    \midrule
    \multirow{4}{*}{50\%} & coding & uniform & 40.9 & 53.4 & 40.3 & 64.1 & 78.9 & 57.8 & 56.4 & 35.0 & 67.1 & 61.6 & 57.7 \\
    & coding & nonunif.\ & \textbf{72.0} & \textbf{69.0} & 35.6 & 55.5 & 77.6 & 54.8 & 54.3 & 34.0 & 64.6 & 60.3 & 54.6 \\
    & general & uniform & 0.0 & 1.8 & \textbf{54.1} & \textbf{77.1} & 83.9 & \textbf{70.9} & \textbf{61.0} & \textbf{42.8} & \textbf{67.5} & \textbf{65.8} & \textbf{65.4} \\
    & general & nonunif.\ & 1.2 & 1.0 & 52.6 & 76.2 & \textbf{84.2} & 70.8 & 59.5 & 41.4 & \textbf{67.5} & 63.5 & 64.4 \\
    \bottomrule
  \end{tabular}
\end{table}

\subsection{Mixed-precision quantization}
\label{sec:quant}

We apply RCO to mixed-precision quantization on Qwen3-8B, assigning one of seven bitwidths (2--8) to each of 252 GPTQ-quantized~\citep{frantar2023gptq} linear layers to minimize calibration KL divergence (Eq.~\ref{eq:kl_loss}) over 256 FineWeb-Edu sequences (seq\_len=2048); evaluation reports perplexity on FineWeb-Edu and C4.

\paragraph{Baselines.}
Table~\ref{tab:quant_main} compares three RCO configurations against EvoPress~\citep{sieberling2024evopress} (evolutionary search, 100 generations), IMPQ~\citep{zhao2025impq} (Shapley surrogate via MILP), and dynamic HIGGS~\citep{malinovskii2024higgs} (linear surrogate via DP) across four average bitwidths (2.25--4.0).
All EvoPress Qwen3-8B numbers are our reproductions under a matched protocol (\citet{sieberling2024evopress} omit Qwen3-8B); details in Appendix~\ref{app:evopress_details}.

At 2.25 bits RCO cuts FW perplexity 36\% below HIGGS (20.47 vs.\ 32.07) and 4\% below EvoPress (vs.\ 21.40) at ${\sim}10{\times}$ lower wall-clock; at 2.5 bits it still surpasses EvoPress (15.43 vs.\ 15.64) at ${\sim}6{\times}$ lower cost. By 3.5--4.0 bits surrogate methods already solve the problem well (HIGGS within 3\% of RCO), and RCO's edge is wall-clock time.

\begin{table*}[!ht]
  \caption{Qwen3-8B mixed-precision quantization. Perplexity ($\downarrow$) on FineWeb-Edu (FW) and C4; FP16: FW=10.96, C4=17.20.
  Same GPTQ-quantized weights per layer/bitwidth.
  \textsuperscript{$\dagger$}Binary bitwidth; Shapley surrogate, MILP.
  \textsuperscript{$\ddagger$}Linear surrogate, DP.}
  \label{tab:quant_main}
  \centering
  \small
  \begin{tabular}{@{}l cc cc cc cc r@{}}
    \toprule
    & \multicolumn{2}{c}{2.25 bits} & \multicolumn{2}{c}{2.5 bits} & \multicolumn{2}{c}{3.5 bits} & \multicolumn{2}{c}{4.0 bits} & \\
    \cmidrule(lr){2-3} \cmidrule(lr){4-5} \cmidrule(lr){6-7} \cmidrule(lr){8-9}
    Method & FW & C4 & FW & C4 & FW & C4 & FW & C4 & Wall \\
    \midrule
    RCO ($g{=}4$, $T{=}200$)  & 20.60 & \textbf{32.44} & 15.78 & 24.73 & \textbf{11.41} & \textbf{17.86} & 11.17 & 17.52 & 62\,m \\
    RCO ($g{=}32$, $T{=}50$)  & 20.66 & 32.94 & \textbf{15.43} & 24.00 & 11.46 & 18.00 & \textbf{11.14} & \textbf{17.40} & 117\,m \\
    RCO ($g{=}16$, $T{=}200$) & \textbf{20.47} & 32.45 & 15.45 & \textbf{23.79} & 11.45 & 17.91 & 11.16 & 17.46 & 215\,m \\
    \midrule
    EvoPress (100 gen)          & 21.40 & 33.92 & 15.64 & 24.63 & 11.50 & 18.07 & 11.16 & 17.58 & 11--14\,h \\
    IMPQ\textsuperscript{$\dagger$} & 24.18 & 38.68 & 18.79 & 29.84 & 12.31 & 19.09 & 11.42 & 17.76 & $\sim$2\,h \\
    HIGGS\textsuperscript{$\ddagger$} & 32.07 & 53.61 & 20.08 & 31.15 & 11.65 & 18.22 & 11.28 & 17.70 & $\sim$2\,h \\
    \bottomrule
  \end{tabular}
\end{table*}

\paragraph{Ablations.}
Appendix~\ref{app:quant_ablations} reports sweeps over nine hyperparameters.
Two findings stand out.
First, Gumbel sample count per step is the most important hyperparameter: increasing from $g{=}1$ to $g{=}4$ reduces FW perplexity by 0.85, with diminishing and non-monotonic gains through $g{=}32$.
Second, sample efficiency dominates step count: $g{=}32$ with 50 steps matches $g{=}16$ with 200 steps at half the forward passes, because the STE gradient benefits more from diverse samples than from additional iterations.
RCO carries no constraint-related hyperparameters for any $q$; the optimization hyperparameters are robust over wide ranges (min.\ temperature 0.005--0.05, lr 0.05--0.2), only $g{=}1$ excepted.

\subsection{Mixed-precision quantization for MoE}
\label{sec:moe_quant}

We compare RCO against MxMoE~\citep{duanmu2025mxmoe}, a sensitivity-driven per-layer ILP, on Qwen1.5-MoE-A2.7B mixed-precision quantization (24 layers, 61 experts/layer; shared GPTQ database, bitwidths 2--8).
Table~\ref{tab:moe_quant_main} reports three expert-bit budgets with uniform-bitwidth references: RCO and RCO* outperform MxMoE on every aggregate metric at every target despite a strictly smaller search space, and at 3.5\,bit RCO matches uniform 4-bit quality, saturating against the FP16 floor (full setup and per-target discussion in Appendix~\ref{app:moe_quant_full}).

\begin{table}[H]
  \caption{Qwen1.5-MoE-A2.7B mixed-precision quantization at 2.5 and 3.5\,bit (3.0\,bit in Table~\ref{tab:moe_quant_full}). PPL ($\downarrow$) on Wikitext-2 (W2), C4, FineWeb-Edu (FW); zero-shot accuracy ($\uparrow$) on six lm-eval tasks. RCO*: attention at FP16 (matches MxMoE); RCO: full search at equal memory. Bold: best per column per budget.}
  \label{tab:moe_quant_main}
  \centering
  \scriptsize
  \setlength{\tabcolsep}{3pt}
  \begin{tabular}{@{}l l cccc ccccccc@{}}
    \toprule
    & & \multicolumn{4}{c}{Perplexity ($\downarrow$)} & \multicolumn{7}{c}{Zero-shot accuracy ($\uparrow$)} \\
    \cmidrule(lr){3-6} \cmidrule(lr){7-13}
    Bits & Method & W2 & C4 & FW & Avg PPL & PIQA & HSwag & ARC-E & ARC-C & Wino & LAMBADA & Avg ZS \\
    \midrule
    FP16 & full model & 6.79 & 10.05 & 9.07 & 8.64 & .805 & .773 & .690 & .445 & .695 & .713 & .687 \\
    \midrule
    \multirow{3}{*}{2.5 bit}
    & MxMoE & 7.90 & 12.44 & 10.28 & 10.21 & .786 & .746 & .659 & \textbf{.416} & .653 & .662 & .653 \\
    & RCO* & 7.47 & 11.57 & 9.81 & 9.61 & \textbf{.793} & .750 & .652 & .411 & .651 & \textbf{.677} & .656 \\
    & RCO  & \textbf{7.40} & \textbf{11.44} & \textbf{9.74} & \textbf{9.53} & .792 & \textbf{.752} & \textbf{.667} & .409 & \textbf{.665} & .676 & \textbf{.660} \\
    \midrule
    \multirow{3}{*}{3.5 bit}
    & MxMoE & 6.94 & 10.35 & 9.23 & 8.84 & \textbf{.801} & .768 & .659 & .433 & \textbf{.686} & .697 & .674 \\
    & RCO* & 6.90 & 10.25 & 9.18 & 8.77 & .799 & .770 & \textbf{.687} & .445 & .673 & .705 & .680 \\
    & RCO  & \textbf{6.88} & \textbf{10.23} & \textbf{9.17} & \textbf{8.76} & \textbf{.801} & \textbf{.771} & .686 & \textbf{.453} & .684 & \textbf{.706} & \textbf{.683} \\
    \midrule
    \multirow{1}{*}{4 bit} 
    & uniform & 6.91 & 10.24 & 9.20 & 8.78 & .809 & .771 & .679 & .440 & .681 & .702 & .680 \\
    \bottomrule
  \end{tabular}
\end{table}

\section{Related Work}
\label{sec:related}

Budget constraints in ML are predominantly enforced via penalty or Lagrangian methods.
In neural architecture search, ProxylessNAS~\citep{cai2019proxyless} and FBNet~\citep{wu2019fbnet} add expected-cost penalties to the objective; MnasNet~\citep{tan2019mnasnet} uses RL reward shaping; in constrained RL, primal-dual methods update Lagrange multipliers alongside policy parameters~\citep{tessler2019reward, stooke2020responsive}.
All satisfy constraints only approximately and require tuning a multiplier or schedule.
RC-DARTS~\citep{jin2020rcdarts} shares the constraint form but enforces it through a decaying Lagrangian penalty, without identifying the manifold structure.
HardCoRe-NAS~\citep{nayman2021hardcore} enforces a hard latency constraint via block-coordinate Frank-Wolfe on a linearized polytope, but restricts the optimizer to Frank-Wolfe updates and does not transport optimizer state.
Riemannian optimization restricts iterates to a smooth manifold to enforce constraints exactly~\citep{absil2008optimization, boumal2023introduction}, with efficient algorithms for Stiefel, Grassmann, SPD, and hyperbolic manifolds~\citep{bonnabel2013stochastic, becigneul2019riemannian, zhang2016first, nickel2017poincare}, and for doubly stochastic matrices via Sinkhorn retraction~\citep{douik2019manifold}.
The probability simplex and its m-flat submanifolds are classical in information geometry~\citep{amari2000methods, lebanon2005riemannian}, typically equipped with the Fisher metric; our Euclidean ambient metric (Section~\ref{sec:method}) is what makes the operations cheap.
The closest prior work, \citet{douik2019manifold}, addresses linear constraints via iterative Sinkhorn retraction; the budget manifold instead admits scalar monotone retraction (Proposition~\ref{prop:retraction}).
To our knowledge, this manifold has not previously been used as an optimization target in this context.

In non-uniform model compression, sensitivity methods~\citep{dong2020hawqv2, yao2021hawqv3, li2023llmmq, helcig2026slq} score layers independently and allocate via DP or ILP, assuming the loss decomposes; HIGGS~\citep{malinovskii2024higgs} formalizes this via a per-layer MSE-to-perplexity linearity theorem, and IMPQ~\citep{zhao2025impq} extends the surrogate with pairwise Shapley interactions.
\citet{sieberling2024evopress} showed the decomposability assumption fails at high compression and proposed evolutionary search over actual model loss; \citet{liu2026evoesap} (EvoESAP) extends this to MoE expert pruning. Both are zero-order and require hundreds of full model evaluations; we compare against EvoPress (Section~\ref{sec:quant}) and EvoESAP (Section~\ref{sec:llm}).
For MoE mixed-precision quantization, MxMoE~\citep{duanmu2025mxmoe} computes per-(expert, weight-block) sensitivities and allocates via a per-layer ILP, again committing to a decomposable surrogate before observing model loss; we compare against MxMoE on Qwen1.5-MoE-A2.7B in Section~\ref{sec:moe_quant}.
Orthogonally, Gumbel-softmax~\citep{jang2017categorical, maddison2017concrete} and the straight-through estimator~\citep{bengio2013estimating} enable gradient flow through discrete samples, and differentiable combinatorial solvers~\citep{berthet2020learning, vlastelica2020differentiation} extend this to structured problems; these address differentiability of discrete choices, not budget enforcement.

Prior analyses of STE bias characterize it via Taylor expansion in the categorical case~\citep{liu2023bridging} or zero-temperature analysis of the Gumbel-Softmax family~\citep{shekhovtsov2023cold}; convergence guarantees for STE training are limited to two-layer networks with binary activations~\citep{yin2019understanding}.
None of these address the projected bias relevant under manifold constraints.

\section{Conclusion and Limitations}
\label{sec:conclusion}

Under softmax parameterization the expected-cost level set is a Riemannian submanifold with closed-form normals, monotone binary-search retraction, and single-inner-product transport. RCO wraps these around Adam to enforce an exact budget at every iterate with no constraint-related hyperparameters, recovering DP-optimal knapsack solutions that Lagrangian methods miss and matching or exceeding evolutionary search on LLM compression at a fraction of the wall-clock cost. The same holds for several simultaneous budgets at one $q{\times}q$ solve per step (Appendix~\ref{app:multi_exp}).

\textbf{Limitations.}
RCO is a heuristic for the discrete problem: it guarantees feasibility (exact for the expected cost, up to the solver's cost grid once decoded; Section~\ref{sec:algorithm}) but not optimality, with quality resting on the Gumbel-STE gradient's uncharacterized tangential bias. Preserved capabilities follow the calibration domain (Table~\ref{tab:coder_next_main}), and the additive-cost retraction leaves non-additive costs such as fused-operator latency to costlier root-finding; DP solver cost is quantified in Appendix~\ref{app:solver_cost}.

\clearpage
\bibliographystyle{iclr2027_conference}
\bibliography{references}

\begin{thebibliography}{43}
\providecommand{\natexlab}[1]{#1}
\providecommand{\url}[1]{\texttt{#1}}
\expandafter\ifx\csname urlstyle\endcsname\relax
  \providecommand{\doi}[1]{doi: #1}\else
  \providecommand{\doi}{doi: \begingroup \urlstyle{rm}\Url}\fi

\bibitem[Absil et~al.(2008)Absil, Mahony, and Sepulchre]{absil2008optimization}
P.-A. Absil, R.~Mahony, and R.~Sepulchre.
\newblock \emph{Optimization Algorithms on Matrix Manifolds}.
\newblock Princeton University Press, 2008.

\bibitem[Amari \& Nagaoka(2000)Amari and Nagaoka]{amari2000methods}
Shun-ichi Amari and Hiroshi Nagaoka.
\newblock \emph{Methods of Information Geometry}, volume 191 of
  \emph{Translations of Mathematical Monographs}.
\newblock American Mathematical Society, 2000.

\bibitem[B{\'e}cigneul \& Ganea(2019)B{\'e}cigneul and
  Ganea]{becigneul2019riemannian}
Gary B{\'e}cigneul and Octavian-Eugen Ganea.
\newblock Riemannian adaptive optimization methods.
\newblock In \emph{International Conference on Learning Representations}, 2019.

\bibitem[Bengio et~al.(2013)Bengio, L{\'e}onard, and
  Courville]{bengio2013estimating}
Yoshua Bengio, Nicholas L{\'e}onard, and Aaron Courville.
\newblock Estimating or propagating gradients through stochastic neurons for
  conditional computation.
\newblock \emph{arXiv preprint arXiv:1308.3432}, 2013.

\bibitem[Berthet et~al.(2020)Berthet, Blondel, Teboul, Cuturi, Vert, and
  Bach]{berthet2020learning}
Quentin Berthet, Mathieu Blondel, Olivier Teboul, Marco Cuturi, Jean-Philippe
  Vert, and Francis Bach.
\newblock Learning with differentiable perturbed optimizers.
\newblock In \emph{Advances in Neural Information Processing Systems},
  volume~33, 2020.

\bibitem[Bonnabel(2013)]{bonnabel2013stochastic}
Silv{\`e}re Bonnabel.
\newblock Stochastic gradient descent on {Riemannian} manifolds.
\newblock \emph{IEEE Transactions on Automatic Control}, 58\penalty0
  (9):\penalty0 2217--2229, 2013.

\bibitem[Boumal(2023)]{boumal2023introduction}
Nicolas Boumal.
\newblock \emph{An Introduction to Optimization on Smooth Manifolds}.
\newblock Cambridge University Press, 2023.

\bibitem[Cai et~al.(2019)Cai, Zhu, and Han]{cai2019proxyless}
Han Cai, Ligeng Zhu, and Song Han.
\newblock {ProxylessNAS}: Direct neural architecture search on target task and
  hardware.
\newblock In \emph{International Conference on Learning Representations}, 2019.

\bibitem[Dong et~al.(2019)Dong, Yao, Gholami, Mahoney, and
  Keutzer]{dong2019hawq}
Zhen Dong, Zhewei Yao, Amir Gholami, Michael~W. Mahoney, and Kurt Keutzer.
\newblock {HAWQ}: {Hessian} aware quantization of neural networks with
  mixed-precision.
\newblock In \emph{Proceedings of the IEEE/CVF International Conference on
  Computer Vision}, pp.\  293--302, 2019.

\bibitem[Dong et~al.(2020)Dong, Yao, Arfeen, Gholami, Mahoney, and
  Keutzer]{dong2020hawqv2}
Zhen Dong, Zhewei Yao, Daiyaan Arfeen, Amir Gholami, Michael~W. Mahoney, and
  Kurt Keutzer.
\newblock {HAWQ-V2}: {Hessian} aware trace-weighted quantization of neural
  networks.
\newblock In \emph{Advances in Neural Information Processing Systems},
  volume~33, pp.\  18518--18529, 2020.

\bibitem[Douik \& Hassibi(2019)Douik and Hassibi]{douik2019manifold}
Ahmed Douik and Babak Hassibi.
\newblock Manifold optimization over the set of doubly stochastic matrices: A
  second-order geometry.
\newblock \emph{IEEE Transactions on Signal Processing}, 67\penalty0
  (22):\penalty0 5761--5774, 2019.

\bibitem[Duanmu et~al.(2025)Duanmu, Li, Yuan, Zheng, Duan, Zhang, and
  Lin]{duanmu2025mxmoe}
Haojie Duanmu, Xiuhong Li, Zhihang Yuan, Size Zheng, Jiangfei Duan, Xingcheng
  Zhang, and Dahua Lin.
\newblock {MxMoE}: Mixed-precision quantization for {MoE} with accuracy and
  performance co-design.
\newblock In \emph{Proceedings of the 42nd International Conference on Machine
  Learning}, volume 267 of \emph{PMLR}, pp.\  14793--14806. PMLR, 2025.

\bibitem[Fedus et~al.(2022)Fedus, Zoph, and
  Shazeer]{Fedus2021SwitchTransformers}
William Fedus, Barret Zoph, and Noam Shazeer.
\newblock {Switch Transformers}: Scaling to trillion parameter models with
  simple and efficient sparsity.
\newblock \emph{Journal of Machine Learning Research}, 23\penalty0
  (120):\penalty0 1--39, 2022.

\bibitem[Frantar \& Alistarh(2022)Frantar and Alistarh]{frantar2022spdy}
Elias Frantar and Dan Alistarh.
\newblock {SPDY}: Accurate pruning with speedup guarantees.
\newblock In \emph{Proceedings of the 39th International Conference on Machine
  Learning}, volume 162 of \emph{PMLR}, pp.\  6726--6743. PMLR, 2022.

\bibitem[Frantar et~al.(2022)Frantar, Ashkboos, Hoefler, and
  Alistarh]{frantar2023gptq}
Elias Frantar, Saleh Ashkboos, Torsten Hoefler, and Dan Alistarh.
\newblock {GPTQ}: Accurate post-training quantization for generative
  pre-trained transformers.
\newblock \emph{arXiv preprint arXiv:2210.17323}, 2022.

\bibitem[Gerganov(2023)]{gerganov2023llamacpp}
Georgi Gerganov.
\newblock llama.cpp: {LLM} inference in {C/C++}, 2023.
\newblock URL \url{https://github.com/ggml-org/llama.cpp}.
\newblock Accessed 2026-08-23.

\bibitem[Helcig et~al.(2026)Helcig, Kurti{\'c}, and Alistarh]{helcig2026slq}
Michael Helcig, Eldar Kurti{\'c}, and Dan Alistarh.
\newblock Statistically-lossless quantization of large language models.
\newblock In \emph{Conference on Language Modeling (COLM)}, 2026.
\newblock arXiv:2605.02404.

\bibitem[Jang et~al.(2017)Jang, Gu, and Poole]{jang2017categorical}
Eric Jang, Shixiang Gu, and Ben Poole.
\newblock Categorical reparameterization with {G}umbel-softmax.
\newblock In \emph{International Conference on Learning Representations}, 2017.

\bibitem[Jin et~al.(2019)Jin, Wang, Slocum, Yang, Dai, Yan, and
  Feng]{jin2020rcdarts}
Xiaojie Jin, Jiang Wang, Joshua Slocum, Ming-Hsuan Yang, Shengyang Dai,
  Shuicheng Yan, and Jiashi Feng.
\newblock {RC-DARTS}: Resource constrained differentiable architecture search.
\newblock \emph{arXiv preprint arXiv:1912.12814}, 2019.

\bibitem[Lasby et~al.(2026)Lasby, Lazarevich, Sinnadurai, Lie, Ioannou, and
  Thangarasa]{lasby2025reap}
Mike Lasby, Ivan Lazarevich, Nish Sinnadurai, Sean Lie, Yani Ioannou, and
  Vithursan Thangarasa.
\newblock {REAP} the experts: Why pruning prevails for one-shot {MoE}
  compression.
\newblock In \emph{The Fourteenth International Conference on Learning
  Representations}, 2026.
\newblock arXiv:2510.13999.

\bibitem[Lebanon(2005)]{lebanon2005riemannian}
Guy Lebanon.
\newblock \emph{Riemannian Geometry and Statistical Machine Learning}.
\newblock PhD thesis, Carnegie Mellon University, 2005.

\bibitem[Li et~al.(2023)Li, Ning, Hong, Liu, Wang, Li, Zhong, Dai, Yang, and
  Wang]{li2023llmmq}
Shiyao Li, Xuefei Ning, Ke~Hong, Tengxuan Liu, Luning Wang, Xiuhong Li, Kai
  Zhong, Guohao Dai, Huazhong Yang, and Yu~Wang.
\newblock {LLM-MQ}: Mixed-precision quantization for efficient {LLM}
  deployment.
\newblock In \emph{Third Workshop on Efficient Natural Language and Speech
  Processing ({ENLSP-III}): Towards the Future of Large Language Models and
  their Emerging Descendants}, 2023.

\bibitem[Liu et~al.(2023)Liu, Dong, Liu, Yu, and Gao]{liu2023bridging}
Liyuan Liu, Chengyu Dong, Xiaodong Liu, Bin Yu, and Jianfeng Gao.
\newblock Bridging discrete and backpropagation: Straight-through and beyond.
\newblock In \emph{Advances in Neural Information Processing Systems},
  volume~36, 2023.

\bibitem[Liu et~al.(2026)Liu, Tang, Sun, Shen, and Yuan]{liu2026evoesap}
Zongfang Liu, Shengkun Tang, Boyang Sun, Zhiqiang Shen, and Xin Yuan.
\newblock {EvoESAP}: Non-uniform expert pruning for sparse {MoE}.
\newblock \emph{arXiv preprint arXiv:2603.06003}, 2026.

\bibitem[Lu et~al.(2024)Lu, Liu, Xu, Zhou, Huang, Zhang, Yan, and
  Li]{lu2024eep}
Xudong Lu, Qi~Liu, Yuhui Xu, Aojun Zhou, Siyuan Huang, Bo~Zhang, Junchi Yan,
  and Hongsheng Li.
\newblock Not all experts are equal: Efficient expert pruning and skipping for
  {Mixture-of-Experts} large language models.
\newblock In \emph{Proceedings of the 62nd Annual Meeting of the Association
  for Computational Linguistics (Volume 1: Long Papers)}, pp.\  6159--6172.
  Association for Computational Linguistics, 2024.

\bibitem[Maddison et~al.(2017)Maddison, Mnih, and Teh]{maddison2017concrete}
Chris~J. Maddison, Andriy Mnih, and Yee~Whye Teh.
\newblock The concrete distribution: A continuous relaxation of discrete random
  variables.
\newblock In \emph{International Conference on Learning Representations}, 2017.

\bibitem[Malinovskii et~al.(2025)Malinovskii, Panferov, Ilin, Guo,
  Richt\'{a}rik, and Alistarh]{malinovskii2024higgs}
Vladimir Malinovskii, Andrei Panferov, Ivan Ilin, Han Guo, Peter Richt\'{a}rik,
  and Dan Alistarh.
\newblock {HIGGS}: Pushing the limits of large language model quantization via
  the linearity theorem.
\newblock In \emph{Proceedings of the 2025 Conference of the Nations of the
  Americas Chapter of the Association for Computational Linguistics: Human
  Language Technologies (Volume 1: Long Papers)}, pp.\  10857--10886.
  Association for Computational Linguistics, 2025.

\bibitem[Nayman et~al.(2021)Nayman, Aflalo, Noy, and
  Zelnik-Manor]{nayman2021hardcore}
Niv Nayman, Yonathan Aflalo, Asaf Noy, and Lihi Zelnik-Manor.
\newblock {HardCoRe-NAS}: Hard constrained diffe{R}entiable neural architecture
  search.
\newblock In \emph{Proceedings of the 38th International Conference on Machine
  Learning}, volume 139 of \emph{PMLR}, pp.\  7979--7990, 2021.

\bibitem[Nickel \& Kiela(2017)Nickel and Kiela]{nickel2017poincare}
Maximilian Nickel and Douwe Kiela.
\newblock Poincar{\'e} embeddings for learning hierarchical representations.
\newblock In \emph{Advances in Neural Information Processing Systems},
  volume~30, 2017.

\bibitem[Shekhovtsov(2023)]{shekhovtsov2023cold}
Alexander Shekhovtsov.
\newblock Cold analysis of {R}ao-{B}lackwellized straight-through
  {G}umbel-softmax gradient estimator.
\newblock In \emph{Proceedings of the 40th International Conference on Machine
  Learning}, volume 202 of \emph{PMLR}, pp.\  30931--30955, 2023.

\bibitem[Sieberling et~al.(2025)Sieberling, Kuznedelev, Kurtic, and
  Alistarh]{sieberling2024evopress}
Oliver Sieberling, Denis Kuznedelev, Eldar Kurtic, and Dan Alistarh.
\newblock {EvoPress}: Accurate dynamic model compression via evolutionary
  search.
\newblock In \emph{Proceedings of the 42nd International Conference on Machine
  Learning}, volume 267 of \emph{PMLR}, pp.\  55556--55590. PMLR, 2025.

\bibitem[Stooke et~al.(2020)Stooke, Achiam, and Abbeel]{stooke2020responsive}
Adam Stooke, Joshua Achiam, and Pieter Abbeel.
\newblock Responsive safety in reinforcement learning by {PID} {Lagrangian}
  methods.
\newblock In \emph{Proceedings of the 37th International Conference on Machine
  Learning}, volume 119 of \emph{PMLR}, pp.\  9133--9143, 2020.

\bibitem[Tan et~al.(2019)Tan, Chen, Pang, Vasudevan, Sandler, Howard, and
  Le]{tan2019mnasnet}
Mingxing Tan, Bo~Chen, Ruoming Pang, Vijay Vasudevan, Mark Sandler, Andrew
  Howard, and Quoc~V. Le.
\newblock {MnasNet}: Platform-aware neural architecture search for mobile.
\newblock In \emph{Proceedings of the IEEE/CVF Conference on Computer Vision
  and Pattern Recognition}, pp.\  2820--2828, 2019.

\bibitem[Tessler et~al.(2019)Tessler, Mankowitz, and Mannor]{tessler2019reward}
Chen Tessler, Daniel~J. Mankowitz, and Shie Mannor.
\newblock Reward constrained policy optimization.
\newblock In \emph{International Conference on Learning Representations}, 2019.

\bibitem[Vlastelica et~al.(2020)Vlastelica, Paulus, Musil, Martius, and
  Rol{\'i}nek]{vlastelica2020differentiation}
Marin Vlastelica, Anselm Paulus, V{\'i}t Musil, Georg Martius, and Michal
  Rol{\'i}nek.
\newblock Differentiation of blackbox combinatorial solvers.
\newblock In \emph{International Conference on Learning Representations}, 2020.

\bibitem[Wainwright \& Jordan(2008)Wainwright and
  Jordan]{wainwright2008graphical}
Martin~J. Wainwright and Michael~I. Jordan.
\newblock Graphical models, exponential families, and variational inference.
\newblock \emph{Foundations and Trends in Machine Learning}, 1\penalty0
  (1--2):\penalty0 1--305, 2008.

\bibitem[Wu et~al.(2019)Wu, Dai, Zhang, Wang, Sun, Wu, Tian, Vajda, Jia, and
  Keutzer]{wu2019fbnet}
Bichen Wu, Xiaoliang Dai, Peizhao Zhang, Yanghan Wang, Fei Sun, Yiming Wu,
  Yuandong Tian, Peter Vajda, Yangqing Jia, and Kurt Keutzer.
\newblock {FBNet}: Hardware-aware efficient {ConvNet} design via differentiable
  neural architecture search.
\newblock In \emph{Proceedings of the IEEE/CVF Conference on Computer Vision
  and Pattern Recognition}, pp.\  10734--10742, 2019.

\bibitem[Yao et~al.(2021)Yao, Dong, Zheng, Gholami, Yu, Tan, Wang, Huang, Wang,
  Mahoney, and Keutzer]{yao2021hawqv3}
Zhewei Yao, Zhen Dong, Zhangcheng Zheng, Amir Gholami, Jiali Yu, Eric Tan,
  Leyuan Wang, Qijing Huang, Yida Wang, Michael~W. Mahoney, and Kurt Keutzer.
\newblock {HAWQ-V3}: Dyadic neural network quantization.
\newblock In \emph{Proceedings of the 38th International Conference on Machine
  Learning}, volume 139 of \emph{PMLR}, pp.\  11875--11886. PMLR, 2021.

\bibitem[Yin et~al.(2024)Yin, Wu, Zhang, Hsieh, Wang, Jia, Li, Jaiswal,
  Pechenizkiy, Liang, Bendersky, Wang, and Liu]{yin2024owl}
Lu~Yin, You Wu, Zhenyu Zhang, Cheng-Yu Hsieh, Yaqing Wang, Yiling Jia, Gen Li,
  Ajay Jaiswal, Mykola Pechenizkiy, Yi~Liang, Michael Bendersky, Zhangyang
  Wang, and Shiwei Liu.
\newblock Outlier weighed layerwise sparsity ({OWL}): A missing secret sauce
  for pruning {LLMs} to high sparsity.
\newblock In \emph{Proceedings of the 41st International Conference on Machine
  Learning}, volume 235 of \emph{PMLR}, pp.\  57101--57115. PMLR, 2024.

\bibitem[Yin et~al.(2019)Yin, Lyu, Zhang, Osher, Qi, and
  Xin]{yin2019understanding}
Penghang Yin, Jiancheng Lyu, Shuai Zhang, Stanley Osher, Yingyong Qi, and Jack
  Xin.
\newblock Understanding straight-through estimator in training activation
  quantized neural nets.
\newblock In \emph{International Conference on Learning Representations}, 2019.

\bibitem[Zhang \& Sra(2016)Zhang and Sra]{zhang2016first}
Hongyi Zhang and Suvrit Sra.
\newblock First-order methods for geodesically convex optimization.
\newblock In \emph{29th Annual Conference on Learning Theory}, volume~49 of
  \emph{PMLR}, pp.\  1617--1638, 2016.

\bibitem[Zhao et~al.(2026)Zhao, Derakhshan, Hyman, Dong, {Abdu Jyothi}, and
  Harris]{zhao2025impq}
Junchen Zhao, Ali Derakhshan, Jayden Hyman, Junhao Dong, Sangeetha {Abdu
  Jyothi}, and Ian Harris.
\newblock {CoopQ}: Cooperative game inspired layerwise mixed precision
  quantization for {LLMs}.
\newblock In \emph{Findings of the Association for Computational Linguistics:
  ACL 2026}, pp.\  7566--7578. Association for Computational Linguistics, 2026.
\newblock Earlier version circulated as {IMPQ}, arXiv:2509.15455.

\bibitem[Zhou et~al.(2022)Zhou, Lei, Liu, Du, Huang, Zhao, Dai, Chen, Le, and
  Laudon]{zhou2022expert}
Yanqi Zhou, Tao Lei, Hanxiao Liu, Nan Du, Yanping Huang, Vincent Zhao, Andrew
  Dai, Zhifeng Chen, Quoc Le, and James Laudon.
\newblock Mixture-of-experts with expert choice routing.
\newblock In \emph{Advances in Neural Information Processing Systems},
  volume~35, pp.\  7103--7114, 2022.

\end{thebibliography}

\clearpage
\appendix
\section{Proofs and Theoretical Details}
\label{app:proofs}

This appendix provides formal statements and proofs for results referenced in Section~\ref{sec:method}, verifies that the algorithm components satisfy standard Riemannian optimization axioms, and derives the extensions to inequality and multiple constraints.

\subsection{Manifold structure and normal vector}
\label{app:manifold_proofs}

\begin{proof}[Proof of Proposition~\ref{prop:manifold}]
  The expected cost $C(\boldsymbol{\alpha}) = \sum_i w_i \inner{\softmax(\boldsymbol{\alpha}_i)}{\mathbf{c}}$ is a composition of the smooth $\softmax$ with a linear function, hence smooth.
  By Proposition~\ref{prop:normal}, the $(i,k)$ entry of $\nabla C(\boldsymbol{\alpha})$ is $w_i\, p_{ik}(c_k - \E_{p_i}[c])$.
  Since $\softmax$ produces strictly positive probabilities ($p_{ik} > 0$ for all $i,k$),
  \[
    \norm{\nabla C}^2
    = \sum_{i=1}^N w_i^2 \sum_{k=1}^K p_{ik}^2 \bigl(c_k - \E_{p_i}[c]\bigr)^2.
  \]
  If the costs are not all equal, then for every group $i$, $\E_{p_i}[c]$ is a strict convex combination of distinct values, so at least one $c_k \neq \E_{p_i}[c]$, and $p_{ik} > 0$ gives $\norm{\nabla C}^2 > 0$.
  Hence $h = C - B$ is a global defining function for $\M$ with $\mathrm{rank}\,Dh = 1$ at every point of $\M$, so $\M$ is a smooth $(NK{-}1)$-dimensional embedded submanifold of $\R^{NK}$ \citep[Definition~3.10 and \S7.7]{boumal2023introduction}, and a Riemannian submanifold under the metric inherited from $\R^{NK}$ \citep[Definition~3.55]{boumal2023introduction}.
\end{proof}

\begin{proof}[Proof of Proposition~\ref{prop:normal}]
  The softmax Jacobian is $\partial p_{ij}/\partial \alpha_{ik} = p_{ij}(\delta_{jk} - p_{ik})$.
  Differentiating~\eqref{eq:cost}:
  \begin{align*}
    \frac{\partial C}{\partial \alpha_{ik}}
    &= w_i \sum_{j=1}^K c_j \, p_{ij}(\delta_{jk} - p_{ik}) \\
    &= w_i \bigl( c_k\, p_{ik} - p_{ik} \textstyle\sum_j c_j\, p_{ij} \bigr) \\
    &= w_i\, p_{ik}\bigl(c_k - \E_{p_i}[c]\bigr). \qedhere
  \end{align*}
\end{proof}

\subsection{Monotonic retraction}
\label{app:retraction_proof}

\begin{proof}[Proof of Proposition~\ref{prop:retraction}]
  Let $\tilde{c}_{ik} = c_k$.
  By the chain rule and Proposition~\ref{prop:normal}:
  \begin{align*}
    \frac{d}{dt}\, C(\boldsymbol{\alpha} + t\tilde{\mathbf{c}})
    &= \sum_{i,k} w_i\, p_{ik}(t)\bigl(c_k - \E_{p_i(t)}[c]\bigr)\, c_k \\
    &= \sum_{i=1}^N w_i \Bigl(\sum_k p_{ik}(t)\, c_k^2 - \bigl(\E_{p_i(t)}[c]\bigr)^2\Bigr) \\
    &= \sum_{i=1}^N w_i\, \Var_{p_i(t)}[c].
  \end{align*}
  Since $w_i > 0$ and $p_{ik}(t) > 0$ for all $k$, $\Var_{p_i(t)}[c] > 0$ whenever costs are not all equal.
  Moreover, as $t \to +\infty$ the softmax concentrates on $\arg\max_k c_k$ giving $C \to \sum_i w_i \max_k c_k$, and as $t \to -\infty$, $C \to \sum_i w_i \min_k c_k$.
  Strict monotonicity and these limiting values ensure that $C(\boldsymbol{\alpha} + t\tilde{\mathbf{c}}) = B$ has a unique solution for any $B$ in the feasible range, found by binary search with exponential convergence in the number of iterations.
\end{proof}

\paragraph{Empirical retraction cost.}
With shift bracket $t \in [-50, 50]$ (range $R = 100$) and tolerance $\varepsilon = 10^{-8}$ on the cost residual, the binary search converges in $\lceil \log_2(R/\varepsilon) \rceil = 34$ iterations in $t$-space.
The cost residual is enforced in expected-cost space, where the local slope $\sum_i w_i \Var_{p_i}[c]$ varies across iterates, so the empirical iteration count differs slightly from this theoretical bound.
Figure~\ref{fig:retraction_iters} reports the distribution across five MCKP scenarios (huge, large, medium, correlated tight, adversarial; $1.8$k--$5.0$k retraction calls each).
Per-call mean ranges from $29.4$ on \textit{adversarial} to $40.8$ on \textit{huge}; the maximum across all scenarios and calls is $45$.

\begin{figure}[H]
  \centering
  \includegraphics[width=0.85\linewidth]{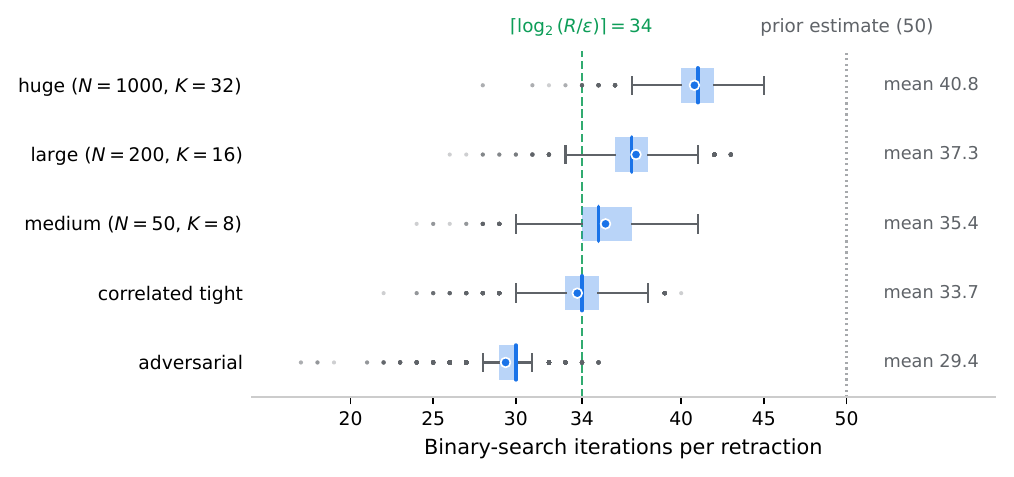}
  \caption{Binary-search iterations per retraction across five MCKP scenarios (boxes: IQR; whiskers: 1.5\,IQR; points: outliers).
  Dashed green line: theoretical bound $\lceil \log_2(R/\varepsilon) \rceil = 34$ for the configuration used in our experiments ($R = 100$, $\varepsilon = 10^{-8}$).
  Dotted gray line: prior loose estimate of $50$.
  All scenarios concentrate near the theoretical bound and stay below $45$ in the worst case.}
  \label{fig:retraction_iters}
\end{figure}

\subsection{Verification of retraction axioms}
\label{app:retraction_axioms}

A retraction on a manifold $\M$ is a smooth map $R \colon T\M \to \M$ satisfying two axioms \citep[Definition~4.1.1]{absil2008optimization}: (i) $R_{\boldsymbol{\alpha}}(\mathbf{0}) = \boldsymbol{\alpha}$ (centering), and (ii) $\frac{d}{dt}R_{\boldsymbol{\alpha}}(t\boldsymbol{\xi})\big|_{t=0} = \boldsymbol{\xi}$ for all $\boldsymbol{\xi} \in T_{\boldsymbol{\alpha}}\M$ (local rigidity).

Our retraction takes a tangent vector $\boldsymbol{\xi} \in T_{\boldsymbol{\alpha}}\M$, forms $\boldsymbol{\alpha}' = \boldsymbol{\alpha} + \boldsymbol{\xi}$, and finds $t^*$ such that $C(\boldsymbol{\alpha}' + t^*\tilde{\mathbf{c}}) = B$.
Smoothness of $t^*$ as a function of $\boldsymbol{\xi}$ follows from the implicit function theorem applied to
\[
  F(\boldsymbol{\xi}, t) = C(\boldsymbol{\alpha} + \boldsymbol{\xi} + t\tilde{\mathbf{c}}) - B.
\]
The partial derivative $\partial F / \partial t = \sum_i w_i\, \Var_{p_i}[c] > 0$ (Proposition~\ref{prop:retraction}) is nonzero everywhere, so $t^*(\boldsymbol{\xi})$ is smooth in a neighborhood of any point on~$\M$.

\begin{proposition}[Retraction axioms]
  \label{prop:retraction_axioms}
  The map $R_{\boldsymbol{\alpha}}(\boldsymbol{\xi}) = \boldsymbol{\alpha} + \boldsymbol{\xi} + t^*(\boldsymbol{\xi})\,\tilde{\mathbf{c}}$, where $t^*$ solves $C(\boldsymbol{\alpha} + \boldsymbol{\xi} + t^*\tilde{\mathbf{c}}) = B$, satisfies both retraction axioms.
\end{proposition}

\begin{proof}
  \textit{Centering.}
  When $\boldsymbol{\xi} = \mathbf{0}$, $C(\boldsymbol{\alpha}) = B$ already holds, so $t^* = 0$ and $R_{\boldsymbol{\alpha}}(\mathbf{0}) = \boldsymbol{\alpha}$.

  \medskip
  \textit{Local rigidity.}
  Write $R_{\boldsymbol{\alpha}}(t\boldsymbol{\xi}) = \boldsymbol{\alpha} + t\boldsymbol{\xi} + t^*(t\boldsymbol{\xi})\,\tilde{\mathbf{c}}$.
  Differentiating the constraint
  \[
    C\bigl(\boldsymbol{\alpha} + t\boldsymbol{\xi} + t^*(t\boldsymbol{\xi})\,\tilde{\mathbf{c}}\bigr) = B
  \]
  at $t = 0$ gives
  \[
    \inner{\nabla C(\boldsymbol{\alpha})}{\boldsymbol{\xi}}
    + \frac{dt^*}{dt}\Big|_{t=0} \inner{\nabla C(\boldsymbol{\alpha})}{\tilde{\mathbf{c}}}
    = 0.
  \]
  The first term vanishes because $\boldsymbol{\xi} \in T_{\boldsymbol{\alpha}}\M$, and $\inner{\nabla C}{\tilde{\mathbf{c}}} \neq 0$ by Proposition~\ref{prop:retraction}, so $\frac{dt^*}{dt}\big|_{t=0} = 0$.
  Therefore $\frac{d}{dt}\big|_{t=0} R_{\boldsymbol{\alpha}}(t\boldsymbol{\xi}) = \boldsymbol{\xi}$.
\end{proof}

\subsection{Riemannian gradient}
\label{app:riemannian_gradient}

The Riemannian gradient of a function $f \colon \M \to \R$ at $\boldsymbol{\alpha} \in \M$ is the unique tangent vector $\mathrm{grad}\, f(\boldsymbol{\alpha}) \in T_{\boldsymbol{\alpha}}\M$ satisfying $\inner{\mathrm{grad}\, f}{\boldsymbol{\xi}} = Df(\boldsymbol{\alpha})[\boldsymbol{\xi}]$ for all $\boldsymbol{\xi} \in T_{\boldsymbol{\alpha}}\M$ \citep[Definition~3.58]{boumal2023introduction}.
We verify that the tangent projection used in Algorithm~\ref{alg:rco} recovers it.

\begin{proposition}[Projected gradient is the Riemannian gradient]
  \label{prop:riemannian_grad}
  Let $L \colon \R^{NK} \to \R$ be smooth and let $\mathbf{n} = \nabla C(\boldsymbol{\alpha})$.
  Then
  \[
    \mathrm{grad}\, L\big|_{\M}(\boldsymbol{\alpha})
    = \nabla L(\boldsymbol{\alpha})
    - \frac{\inner{\nabla L(\boldsymbol{\alpha})}{\mathbf{n}}}{\norm{\mathbf{n}}^2}\;\mathbf{n}.
  \]
\end{proposition}

\begin{proof}
  The tangent space is $T_{\boldsymbol{\alpha}}\M = \{\boldsymbol{\xi} : \inner{\mathbf{n}}{\boldsymbol{\xi}} = 0\}$.
  Since $\M$ inherits the Euclidean metric, the Riemannian gradient equals the orthogonal projection of $\nabla L$ onto $T_{\boldsymbol{\alpha}}\M$.
  The normal space is $\mathrm{span}(\mathbf{n})$, giving the stated formula.
\end{proof}

This confirms that every projected-gradient step in Algorithm~\ref{alg:rco} moves in the steepest-descent direction on $\M$.
Note that this differs from projected gradient descent (PGD), which projects the \emph{iterate} onto a convex constraint set after each step.
Here we project the \emph{gradient} onto the tangent space and use a separate retraction to bring the iterate back to $\M$.
This is necessary because $\M$ is not convex, so projecting iterates onto it is not well-defined in general.

\subsection{Vector transport by projection}
\label{app:transport}

Exact parallel transport on implicitly defined submanifolds is expensive to compute. We instead use \emph{vector transport by projection} \citep[Section~8.1.3]{absil2008optimization}: given $\boldsymbol{\xi} \in T_{\boldsymbol{\alpha}}\M$ and a retraction step to $\boldsymbol{\alpha}'$, the transported vector is
\[
  \mathcal{T}_{\boldsymbol{\alpha} \to \boldsymbol{\alpha}'}(\boldsymbol{\xi})
  = \boldsymbol{\xi}
  - \frac{\inner{\boldsymbol{\xi}}{\mathbf{n}'}}{\norm{\mathbf{n}'}^2}\;\mathbf{n}',
\]
where $\mathbf{n}' = \nabla C(\boldsymbol{\alpha}')$.
This satisfies the vector transport axioms \citep[Definition~8.1.1]{absil2008optimization} and coincides with parallel transport to first order in the step size.

In Algorithm~\ref{alg:rco}, we apply this transport to Adam's first moment $\mathbf{m}$. We only transport the first moment; the second moment is a scalar scaling factor per coordinate and does not need transport.

\subsection{Robustness to gradient bias}
\label{app:bias}

The convergence analysis of Riemannian gradient descent (Section~\ref{sec:method}) assumes exact gradients.
In practice, Algorithm~\ref{alg:rco} uses biased gradient estimates from the straight-through estimator.
The following result shows that the tangent projection removes the normal component of any bias before it reaches the optimizer, eliminating the first-order constraint violation that unprojected bias would cause; the remaining per-step violation is second-order, scaling with manifold curvature times $\norm{P_T(\tilde{\mathbf{g}})}^2$ (which still depends on the tangential bias) and is corrected exactly by the retraction at every step.

\begin{proposition}[Constraint robustness under biased gradients]
  \label{prop:bias}
  Let $\tilde{\mathbf{g}} = \nabla L(\boldsymbol{\alpha}) + \mathbf{b}$ be a biased gradient estimate at $\boldsymbol{\alpha} \in \M$, where $\mathbf{b}$ is an arbitrary bias vector.
  Let $\mathbf{n} = \nabla C(\boldsymbol{\alpha})$ denote the constraint normal and write $P_T(\mathbf{v}) = \mathbf{v} - (\inner{\mathbf{v}}{\mathbf{n}} / \norm{\mathbf{n}}^2)\,\mathbf{n}$ for the tangent projection.
  \begin{enumerate}
    \item[\textnormal{(i)}] \textbf{Normal bias elimination.}
    The tangent projection eliminates the constraint-normal component of the bias:
    \[
      P_T(\tilde{\mathbf{g}}) = P_T(\nabla L) + \mathbf{b}_T,
    \]
    where $\mathbf{b}_T = \mathbf{b} - \frac{\inner{\mathbf{b}}{\mathbf{n}}}{\norm{\mathbf{n}}^2}\,\mathbf{n}$ is the tangential component of the bias.
    The normal component $\mathbf{b}_n = \mathbf{b} - \mathbf{b}_T$ does not enter the optimizer.

    \item[\textnormal{(ii)}] \textbf{Constraint violation bound (gradient descent).}
    For a gradient step $\boldsymbol{\alpha}' = \boldsymbol{\alpha} - \eta\, P_T(\tilde{\mathbf{g}})$, the per-step constraint violation satisfies
    \[
      |C(\boldsymbol{\alpha}') - B|
      \;\leq\; \frac{\eta^2}{2}\, \norm{\nabla^2 C(\boldsymbol{\alpha})}_{\mathrm{op}}\, \norm{P_T(\tilde{\mathbf{g}})}^2 + O(\eta^3).
    \]
    This bound is second-order in $\eta$ and independent of $\norm{\mathbf{b}_n}$.
    Without projection, the violation is first-order: $|C(\boldsymbol{\alpha} - \eta\,\tilde{\mathbf{g}}) - B| = \eta\,|\inner{\mathbf{n}}{\tilde{\mathbf{g}}}| + O(\eta^2)$.

    \item[\textnormal{(iii)}] \textbf{Momentum insulation (Adam).}
    In Algorithm~\ref{alg:rco}, after tangent projection (line~10) and vector transport (line~14), Adam's first moment satisfies $\inner{\mathbf{m}_t}{\mathbf{n}_t} = 0$ for all $t \geq 0$, regardless of the bias in the gradient estimates.
  \end{enumerate}
\end{proposition}

\begin{proof}
  \textit{Part~(i).}
  $P_T$ is linear (it is an orthogonal projection onto the hyperplane $\{\mathbf{v} : \inner{\mathbf{v}}{\mathbf{n}} = 0\}$), so $P_T(\tilde{\mathbf{g}}) = P_T(\nabla L) + P_T(\mathbf{b})$.
  Since $P_T(\mathbf{b}) = \mathbf{b} - \frac{\inner{\mathbf{b}}{\mathbf{n}}}{\norm{\mathbf{n}}^2}\,\mathbf{n} = \mathbf{b}_T$, the normal component $\mathbf{b}_n$ is removed.

  \medskip
  \textit{Part~(ii).}
  Let $\mathbf{d} = P_T(\tilde{\mathbf{g}})$.
  Taylor-expanding $C$ at $\boldsymbol{\alpha} \in \M$:
  \[
    C(\boldsymbol{\alpha} - \eta\,\mathbf{d})
    = \underbrace{C(\boldsymbol{\alpha})}_{=\,B}
    - \eta\,\underbrace{\inner{\nabla C(\boldsymbol{\alpha})}{\mathbf{d}}}_{=\,\inner{\mathbf{n}}{\mathbf{d}}\,=\,0}
    + \frac{\eta^2}{2}\,\mathbf{d}^\top \nabla^2 C(\boldsymbol{\alpha})\,\mathbf{d}
    + O(\eta^3).
  \]
  The first-order term vanishes because $\mathbf{d} \in T_{\boldsymbol{\alpha}}\M$.
  The Hessian $\nabla^2 C$ exists by smoothness of $C$ (Proposition~\ref{prop:manifold}); its operator norm is bounded because every entry is a polynomial in the softmax probabilities $p_{ik} \in [0,1]$ with coefficients depending only on $w_i$ and $c_k$, giving the stated bound.
  Without projection, $\mathbf{d} = \tilde{\mathbf{g}}$ and $\inner{\mathbf{n}}{\tilde{\mathbf{g}}} \neq 0$ in general, so the violation is $O(\eta)$.

  \medskip
  \textit{Part~(iii).}
  By induction on $t$.
  At each step, the momentum passes through two stages: the Adam update (which mixes old momentum with the new projected gradient) and the vector transport (which re-projects onto the tangent space at the retracted point).
  We write $\mathbf{m}_t^+$ for the momentum after the Adam update and $\mathbf{m}_t^{\mathcal{T}}$ for the momentum after transport.

  \textit{Base case.}
  $\mathbf{m}_0^{\mathcal{T}} = \mathbf{0}$, so $\inner{\mathbf{m}_0^{\mathcal{T}}}{\mathbf{n}} = 0$ for any $\mathbf{n}$.

  \textit{Inductive step.}
  Suppose $\inner{\mathbf{m}_{t-1}^{\mathcal{T}}}{\mathbf{n}_t} = 0$, where $\mathbf{n}_t = \nabla C(\boldsymbol{\alpha}_t)$ is the normal at the current iterate.
  The projected gradient satisfies $\inner{P_T(\tilde{\mathbf{g}}_t)}{\mathbf{n}_t} = 0$ by construction.
  Adam's momentum update is a convex combination of these two orthogonal-to-$\mathbf{n}_t$ vectors:
  $\mathbf{m}_t^+ = \beta_1 \mathbf{m}_{t-1}^{\mathcal{T}} + (1-\beta_1)\, P_T(\tilde{\mathbf{g}}_t)$,
  so $\inner{\mathbf{m}_t^+}{\mathbf{n}_t} = \beta_1 \cdot 0 + (1-\beta_1) \cdot 0 = 0$.
  After the Adam step and retraction to $\boldsymbol{\alpha}_t' \in \M$, vector transport (Algorithm~\ref{alg:rco}, line~14) projects $\mathbf{m}_t^+$ onto the new tangent space:
  $\mathbf{m}_t^{\mathcal{T}} = \mathbf{m}_t^+ - (\inner{\mathbf{m}_t^+}{\mathbf{n}_t'}/\norm{\mathbf{n}_t'}^2)\,\mathbf{n}_t'$,
  which satisfies $\inner{\mathbf{m}_t^{\mathcal{T}}}{\mathbf{n}_t'} = 0$ by construction of the projection.
  Since $\boldsymbol{\alpha}_{t+1} = \boldsymbol{\alpha}_t'$, we have $\mathbf{n}_{t+1} = \mathbf{n}_t'$, completing the induction.
\end{proof}

\begin{remark}[Adam's per-coordinate scaling]
  \label{rem:adam_bias}
  Part~(ii) applies to gradient descent.
  Adam's step direction is $\mathbf{d} = \hat{\mathbf{m}} \oslash (\sqrt{\hat{\mathbf{v}}} + \epsilon)$ (element-wise), and the per-coordinate scaling does not preserve the tangent plane: $\inner{\mathbf{n}}{\mathbf{d}} \neq 0$ even when $\inner{\mathbf{m}}{\mathbf{n}} = 0$.
  To quantify how much normal leakage Adam introduces, write $D = \mathrm{diag}(1/(\sqrt{\hat{v}_{ik}} + \epsilon))$ for the scaling matrix and $\bar{d}$ for its mean diagonal entry, and decompose $D$ into its uniform part $\bar{d}I$ (which preserves the tangent plane) and the non-uniform residual $D - \bar{d}I$:
  \[
    \inner{\mathbf{n}}{D\mathbf{m}}
    = \mathbf{m}^\top(D - \bar{d}I)\,\mathbf{n}
    + \bar{d}\;\underbrace{\inner{\mathbf{m}}{\mathbf{n}}}_{=\,0\;\text{by (iii)}}.
  \]
  The projection eliminates the direct bias-driven normal component $\bar{d}\,\inner{\mathbf{m}}{\mathbf{n}}$.
  The remaining normal leakage $|\mathbf{m}^\top(D - \bar{d}I)\,\mathbf{n}| \leq \norm{\mathbf{m}}\,\norm{\mathbf{n}}\,\norm{D - \bar{d}I}_{\mathrm{op}}$ is proportional to the non-uniformity of Adam's second-moment estimates: if $D$ were a scalar matrix (uniform scaling), this term would vanish regardless of $\norm{\mathbf{m}}$.
  The factor $\norm{\mathbf{m}}$ still reflects the accumulated (tangential) gradient history, including tangential bias; what the projection guarantees is that bias enters this bound only through the non-uniformity factor $\norm{D - \bar{d}I}_{\mathrm{op}}$, not directly.
  The retraction (Proposition~\ref{prop:retraction}) corrects this residual exactly.
\end{remark}

Proposition~\ref{prop:bias} establishes that the tangent projection and vector transport isolate constraint satisfaction from gradient estimation quality: the constraint-normal component of any bias is eliminated before it reaches the optimizer.
The tangential component $\mathbf{b}_T$ does not cause constraint violation, but it does affect optimization quality.

\subsection{Inequality constraints via slack variables}
\label{app:slack}

For the inequality constraint $C(\boldsymbol{\alpha}) \leq B$, we introduce a scalar slack variable $s \in \R$ and define the augmented equality constraint
\begin{equation}
  \label{eq:slack_app}
  \hat{C}(\boldsymbol{\alpha}, s) = C(\boldsymbol{\alpha}) + s^2 = B.
\end{equation}
This defines a smooth manifold $\hat{\M} \subset \R^{NK+1}$.
The gradient of $\hat{C}$ in the augmented space is
\[
  \nabla \hat{C}(\boldsymbol{\alpha}, s)
  = \bigl(\nabla_{\boldsymbol{\alpha}} C(\boldsymbol{\alpha}),\; 2s\bigr).
\]
When $s > 0$, the normal has a nonzero component in the $s$-direction, so the tangent space includes directions that change $\boldsymbol{\alpha}$ freely (with compensating changes in $s$).
The optimizer can decrease cost below budget by increasing $s$.
As $s \to 0$, the algorithm reduces to the equality-constrained case.

The tangent projection in the augmented space is
\[
  \mathrm{proj}\bigl(\mathbf{g},\, 0\bigr)
  = \bigl(\mathbf{g},\, 0\bigr)
  - \frac{\inner{\mathbf{g}}{\nabla_{\boldsymbol{\alpha}} C}}
         {\norm{\nabla_{\boldsymbol{\alpha}} C}^2 + 4s^2}
  \;\bigl(\nabla_{\boldsymbol{\alpha}} C,\, 2s\bigr),
\]
where $\mathbf{g} = \nabla_{\boldsymbol{\alpha}} L$ and the $s$-component of the objective gradient is zero (the objective does not depend on~$s$).

Retraction proceeds by first updating $(\boldsymbol{\alpha}, s)$ via the projected Adam step, then adjusting: if $C(\boldsymbol{\alpha}) > B$, retract $\boldsymbol{\alpha}$ via binary search to $C(\boldsymbol{\alpha}) = B$ and set $s = 0$; otherwise set $s = \sqrt{B - C(\boldsymbol{\alpha})}$.
This maintains $\hat{C}(\boldsymbol{\alpha}, s) = B$ exactly.

\paragraph{Smoothness at $s = 0$.}
The retraction is continuous but not smooth at the boundary $s = 0$, because the square-root branch $s' = \sqrt{B - C(\boldsymbol{\alpha}')}$ has an infinite derivative as $s' \to 0^+$.
In practice, the optimizer either settles to a point with $s > 0$ (the optimum is strictly under budget) or converges to the equality-constrained manifold ($s = 0$), spending at most a few iterations near the boundary.
We observe no numerical issues from this non-smoothness in any of our experiments.

\subsection{Multiple equality constraints}
\label{app:multi}

Given $q$ budgets
\[
  C_j(\boldsymbol{\alpha}) = \sum_{i=1}^N w_{ji}\, \inner{\mathbf{p}_i}{\mathbf{c}_j} = b_j, \qquad j = 1, \ldots, q,
\]
where $w_{ji} \geq 0$ is the weight of group $i$ in budget $j$ (zero if the group does not enter it) and $\mathbf{c}_j \in \R^K$ collects the option costs of budget $j$, the feasible set is $\M = \bigcap_j C_j^{-1}(b_j)$.
Each budget has a closed-form normal by Proposition~\ref{prop:normal}, $(\mathbf{n}_j)_{ik} = w_{ji}\, p_{ik}\bigl(c_{jk} - \E_{p_i}[\mathbf{c}_j]\bigr)$, and we collect $\mathbf{N} = [\mathbf{n}_1 \cdots \mathbf{n}_q]$.
The normals are linearly independent exactly when no nonzero $\boldsymbol{\lambda} \in \R^q$ makes $\mathbf{v}_i = \sum_j \lambda_j w_{ji}\mathbf{c}_j$ constant across the options of every group $i$, since $\bigl(\sum_j \lambda_j \mathbf{n}_j\bigr)_{ik} = p_{ik}\bigl(v_{ik} - \E_{p_i}[\mathbf{v}_i]\bigr)$ with $p_{ik} > 0$.
The condition does not depend on $\boldsymbol{\alpha}$; where it holds, $\M$ is a smooth $(NK - q)$-dimensional submanifold with tangent space $T_{\boldsymbol{\alpha}}\M = \{\boldsymbol{\xi} : \mathbf{N}^\top \boldsymbol{\xi} = \mathbf{0}\}$.
Both settings of Appendix~\ref{app:multi_exp} satisfy it: per-layer expert counts act on disjoint groups with non-constant costs, and for a joint sparsity and bitwidth budget over the options (prune, $2, 3, 4, 5$ bits), $\lambda_1 w_{1i}(1,0,0,0,0) + \lambda_2 w_{2i}(0,2,3,4,5)$ is constant only at $\boldsymbol{\lambda} = \mathbf{0}$.
The tangent projection becomes
\[
  \mathbf{g}_{\mathrm{tan}}
  = \mathbf{g}
  - \mathbf{N}\bigl(\mathbf{N}^\top\mathbf{N}\bigr)^{-1}\mathbf{N}^\top\mathbf{g},
\]
which requires solving a $q \times q$ system, cheap when $q \ll NK$; momentum transport applies the same operator.

Retraction generalizes via Newton root-finding on the $q$-dimensional system: shift each group along its active cost vectors, $\boldsymbol{\alpha}_i \mapsto \boldsymbol{\alpha}_i + \sum_{l : w_{li} > 0} t_l\, \mathbf{c}_l$, and find $\mathbf{t} \in \R^q$ such that $C_j = b_j$ for all $j$.
The Jacobian entry is
\[
  J_{jl} = \sum_{i : w_{li} > 0} w_{ji}\, \operatorname{Cov}_{p_i}[\mathbf{c}_j,\, \mathbf{c}_l],
\]
available in closed form from the current softmax distributions.
For $q = 1$ this reduces to the scalar $\sum_i w_i\,\Var_{p_i}[\mathbf{c}]$ from Proposition~\ref{prop:retraction}, and binary search on the same monotone function is the simpler alternative used in Section~\ref{sec:method}.
Quadratic convergence typically requires 2--4 iterations for general~$q$.
When the budgets act on disjoint groups, $\mathbf{N}^\top\mathbf{N}$ and $\mathbf{J}$ are diagonal and each step reduces to $q$ independent binary searches, one per budget.
For two budgets on shared groups, as for sparsity and bitwidth, $\mathbf{J}$ is nonsingular whenever both weights are constant across the shared groups and independence holds there: then $\mathbf{J} = \mathbf{A}\mathbf{M} + \mathbf{D}$ with $\mathbf{A} \succ 0$ diagonal, $\mathbf{M}$ the positive definite covariance Gram matrix of the two cost vectors over the shared groups, and $\mathbf{D} \succeq 0$ diagonal from the groups in one budget, so $\mathbf{J}$ is similar to $\mathbf{A}^{1/2}\mathbf{M}\mathbf{A}^{1/2} + \mathbf{D} \succ 0$.

Figure~\ref{fig:multi_constraint} validates this extension on a synthetic problem with $q{=}16$ simultaneous resource constraints, $N{=}500$ groups, and $K{=}32$ options.
The budget manifold satisfies all 16 constraints to $\sim\!10^{-13}$ (the limit of double-precision arithmetic) throughout optimization.
Lagrangian penalty cannot balance 16 independent multipliers and diverges to violations of order~10; augmented Lagrangian reduces this to $\sim\!10^{-2}$ but still exceeds the budget manifold by ten orders of magnitude.

\begin{figure}[h]
  \centering
  \includegraphics[width=\linewidth]{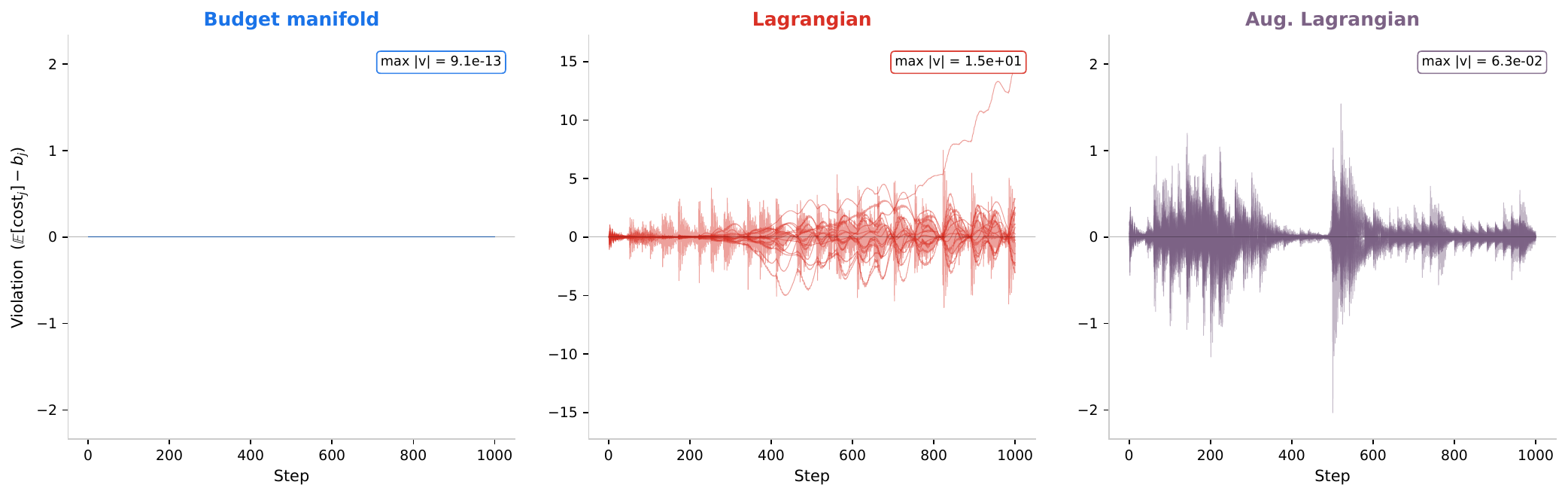}
  \caption{Per-constraint violation trajectories with $q{=}16$ simultaneous budget constraints ($N{=}500$, $K{=}32$).
  Each trace is one constraint; y-axis scales differ between panels.
  Budget manifold enforces all 16 constraints to $\sim\!10^{-13}$ throughout optimization.
  Lagrangian methods lack the geometric structure to balance many constraints and exhibit persistent violations orders of magnitude larger.}
  \label{fig:multi_constraint}
\end{figure}

\section{Application to Non-uniform LLM Compression}
\label{app:llm_setup}

We apply RCO to non-uniform LLM compression, where each network layer (or expert) is a group and the options are compression levels (bitwidths, sparsity rates, or keep/prune).

\paragraph{Objective.}
The loss $L(\mathbf{z}^*)$ in Algorithm~\ref{alg:rco} is the KL divergence between the full-precision model and the model under assignment $\mathbf{z}^*$, evaluated on a calibration set $\mathcal{D}$:
\begin{equation}
  \label{eq:kl_loss}
  L(\mathbf{z}^*)
  \;=\; \frac{1}{|\mathcal{D}|}
    \sum_{x \in \mathcal{D}}
    \mathrm{KL}\!\bigl(
      p_{\mathrm{ref}}(\cdot \mid x)
      \;\big\|\;
      p_{\mathbf{z}^*}(\cdot \mid x)
    \bigr),
\end{equation}
where $p_{\mathrm{ref}}$ is the full-precision model's next-token distribution and $p_{\mathbf{z}^*}$ is the distribution under configuration $\mathbf{z}^*$.
Reference log-probabilities are computed once and cached.
Because $L$ is non-decomposable (Section~\ref{sec:intro}), minimizing it directly is what distinguishes this approach from sensitivity methods that optimize per-group proxy scores.

\paragraph{From assignments to weights.}
Each forward pass evaluates the model under $\mathbf{z}^*$ from Eq.~\eqref{eq:dp}: for MoE expert pruning, a per-expert (keep, prune) mask at costs $(0, 1)$ that zeros pruned outputs and leaves router logits unchanged; for mixed-precision quantization, a one-hot bitwidth selector that assembles $\mathbf{W}_i = \mathbf{W}_i^{\mathrm{ref}} + \sum_k z^*_{ik} \boldsymbol{\Delta}_{ik}$ from pre-computed GPTQ residuals $\boldsymbol{\Delta}_{ik} = \mathbf{W}_i^{(k)} - \mathbf{W}_i^{\mathrm{ref}}$~\citep{frantar2023gptq}.
DP costs and budget are taken in integer units (bitwidths scaled by parameter fraction for quantization, binary keep/prune for expert pruning).
Algorithm~\ref{alg:rco} is agnostic to grouping and to the initial logits $\boldsymbol{\alpha}_0$: any partition of compressible units, any per-group option set, and any prior on $\boldsymbol{\alpha}_0$ leave the iteration unchanged, with feasibility absorbed by the retraction-based initialization step (Section~\ref{sec:application}).
Sensitivity scores from prior work, such as expert-activation frequencies for expert pruning or layer-output L2 sensitivities for quantization, serve as valid initializations alongside uniform $\boldsymbol{\alpha}_0 = \mathbf{0}$.
The STE replaces $z^*_{ik}$ with $\hat{p}_{ik} = \softmax(\hat{\boldsymbol{\alpha}}_i)_k$ in the backward pass, so gradients flow through the softmax Jacobian at the perturbed logits.

\subsection{Cost of the discrete solver}
\label{app:solver_cost}

The budget-constrained DP of Eq.~\eqref{eq:dp} runs once per Gumbel sample per step.
Table~\ref{tab:solver_cost} benchmarks the pure-Python DP used in our experiments, run on a CPU at a grid of $500$ ticks per average bit ($B' = 500 \times$ target average bits).
The DP time grows linearly in $N$.
For the grouped Qwen3-8B setting, an in-process measurement gives about $14$\,s per step with DP decoding in the forward pass versus about $12$\,s with a plain argmax, i.e., about $500$\,ms per DP call including CPU--GPU synchronization, about $14\%$ of the step; the table has $NKB' = 144 \times 4 \times 1500 = 864{,}000$ entries.
These per-call times reflect the Python implementation; no compiled version was measured.
For expert pruning, where the costs are $0$ (keep) and $1$ (prune) and the budget fixes how many are pruned, the discrete problem is a top-$k$ over the perturbed logit differences (Section~\ref{sec:algorithm}) and was not timed separately; for scale, one Qwen3-Coder-Next search step ($N = 24{,}576$ experts, $g = 4$) takes about $47$\,s.

\begin{table}[!ht]
  \caption{Discrete-solver cost. Pure-Python DP on a CPU at $500$ ticks per average bit, $K{=}4$ options in all rows. DP: time per call; GPU: forward+backward time per sample, shown with the DP share only where measured (4$\times$A100, batch 1, 4096-token sequences). The Qwen3-8B row is the grouped setting with bitwidths $2$--$5$ (36 layers $\times$ 4 groups), not the configuration of Table~\ref{tab:quant_main}; OLMoE-1B-7B is per-expert quantization.}
  \label{tab:solver_cost}
  \centering
  \small
  \begin{tabular}{@{}lrrrrr@{}}
    \toprule
    Setting & $N$ & $B'$ & DP (ms) & GPU (ms) & DP share \\
    \midrule
    Qwen3-8B    & 144  & 1500 & 270  & 3000 & 8.3\% \\
    OLMoE-1B-7B & 1040 & 1250 & 1783 & --   & --    \\
    Synthetic   & 352  & 1500 & 656  & --   & --    \\
    Synthetic   & 555  & 1500 & 1157 & --   & --    \\
    \bottomrule
  \end{tabular}
\end{table}

\section{Full Expert Pruning Results}
\label{app:expert_pruning_full}

\paragraph{Gradient signal under sparse routing.}
In MoE models, only the top-$k$ routed experts are computed per token; the STE pruning mask is applied after routing by scaling each selected expert's contribution.
An expert's pruning logit $\alpha_i$ therefore receives gradient only from tokens where expert $i$ was routed, which is a small fraction of the calibration set.
This makes the per-expert gradient signal inherently sparser than in the quantization setting, where every layer processes every token.
The sparsity explains why the Gumbel sample count $g$ is the most important hyperparameter in the expert pruning ablations: more samples per step expose each expert to more diverse pruning contexts, compensating for the low per-token coverage.
Because router weights are frozen, pruned experts are masked out at inference in the same way as during optimization (the remaining top-$k$ experts are upweighted); there is no approximation gap between the STE training objective and deployment.

\subsection{OLMoE-1B-7B: sparsity sweep}

Table~\ref{tab:olmoe_sparsity_full} reports all eight benchmarks across sparsity levels from 5\% to 50\% (RCO, 300 steps).

\begin{table}[H]
  \caption{OLMoE-1B-7B per-benchmark results across sparsity levels (RCO, 300 steps).
  Avg is the unweighted mean of all eight benchmarks.}
  \label{tab:olmoe_sparsity_full}
  \centering
  \small
  \setlength{\tabcolsep}{4pt}
  \begin{tabular}{@{}l ccccccc@{}}
    \toprule
    & Full & 5\% & 10\% & 15\% & 20\% & 25\% & 50\% \\
    \midrule
    ARC-C & .493 & .497 & .489 & .497 & .470 & .458 & .334 \\
    ARC-E & .758 & .760 & .741 & .738 & .718 & .698 & .525 \\
    BoolQ & .768 & .757 & .750 & .717 & .707 & .713 & .602 \\
    HellaSwag & .806 & .797 & .779 & .759 & .735 & .711 & .469 \\
    MMLU & .534 & .531 & .526 & .506 & .495 & .479 & .301 \\
    OBQA & .468 & .470 & .466 & .462 & .450 & .456 & .326 \\
    RTE & .715 & .726 & .693 & .704 & .718 & .755 & .570 \\
    WinoGrande & .680 & .690 & .650 & .645 & .624 & .596 & .533 \\
    \midrule
    \textbf{Avg} & \textbf{.653} & \textbf{.653} & \textbf{.637} & \textbf{.629} & \textbf{.615} & \textbf{.608} & \textbf{.458} \\
    \bottomrule
  \end{tabular}
\end{table}

At 5\% sparsity, every individual benchmark remains within 1--2 percentage points of the full model; several (ARC-C, ARC-E, OBQA, RTE) are within noise.
Degradation is monotonic for most benchmarks, with the exception of RTE, which improves under moderate pruning (0.715 $\to$ 0.755 at 25\%).
The largest per-benchmark drops occur at 50\% sparsity on HellaSwag ($-$33.7 points) and MMLU ($-$23.3 points), consistent with these benchmarks' sensitivity to knowledge capacity.

\subsection{OLMoE-1B-7B: full iteration sweep}

Table~\ref{tab:olmoe_iter_full} reports the full per-benchmark iteration sweep on OLMoE-1B-7B at 25\% sparsity.

\begin{table}[H]
  \caption{OLMoE-1B-7B at 25\% sparsity: full per-benchmark iteration sweep.
  All eight benchmarks contributing to Avg.}
  \label{tab:olmoe_iter_full}
  \centering
  \small
  \setlength{\tabcolsep}{3.5pt}
  \begin{tabular}{@{}l cccccccc c r@{}}
    \toprule
    & ARC-C & ARC-E & BoolQ & HSwag & MMLU & OBQA & RTE & Wino & \textbf{Avg} & Time \\
    \midrule
    Full model & .493 & .758 & .768 & .806 & .534 & .468 & .715 & .680 & .653 & - \\
    \midrule
    REAP & .408 & .676 & .683 & .660 & .452 & .406 & .664 & .580 & .566 & {$<$1\,m} \\
    EvoESAP & .427 & .641 & .712 & .685 & .474 & .386 & .682 & \textbf{.639} & .581 & 5.8\,h \\
    \midrule
    RCO ($T{=}50$) & .442 & .713 & .693 & .691 & .476 & .440 & .718 & .602 & .597 & $\sim$10\,m \\
    RCO ($T{=}100$) & .446 & .716 & .692 & .713 & .465 & .438 & .722 & .608 & .600 & $\sim$23\,m \\
    RCO ($T{=}150$) & .454 & .727 & .697 & .703 & .427 & .452 & .668 & .616 & .593 & $\sim$32\,m \\
    RCO ($T{=}200$) & .459 & .701 & .701 & .708 & .475 & .420 & .697 & .608 & .596 & $\sim$44\,m \\
    RCO ($T{=}300$) & \textbf{.458} & \textbf{.698} & \textbf{.713} & \textbf{.711} & \textbf{.479} & \textbf{.456} & \textbf{.755} & .596 & \textbf{.608} & $\sim$66\,m \\
    RCO ($T{=}500$) & .456 & .699 & .698 & .711 & .479 & .444 & .722 & .612 & .603 & $\sim$108\,m \\
    \bottomrule
  \end{tabular}
\end{table}

RCO reaches 0.597 Avg at 50 steps and improves to 0.608 at 300 steps.
Returns diminish beyond 300 steps (500 steps: 0.603).
Individual benchmarks show non-monotonic behavior during optimization (e.g., MMLU dips at $T{=}150$), reflecting the stochastic nature of Gumbel-STE sampling, but the aggregate Avg trends upward through 300 steps.

\subsection{Qwen3-30B-A3B: full per-benchmark comparison}

Table~\ref{tab:qwen3_full} reports the full per-benchmark breakdown of RCO vs.\ EvoESAP on Qwen3-30B-A3B at 25\% and 50\% sparsity (summarized in Figure~\ref{fig:summary}).

\begin{table}[H]
  \caption{Qwen3-30B-A3B per-expert pruning: full benchmark breakdown.
  RCO attains the best Avg at both sparsities, winning 7/8 benchmarks at 25\% against our EvoESAP reproduction (WinoGrande excepted) and 6/8 at 50\%.
  At 25\%, EvoESAP is shown both as reported by \citet{liu2026evoesap} and as our reproduction of their search (Appendix~\ref{app:evoesap_details}); REAP uniform is uniform allocation under the same REAP saliency, also ours.
  Largest gains at 25\% over the reproduction: RTE (+3.2), HellaSwag (+2.7), OBQA (+1.6).}
  \label{tab:qwen3_full}
  \centering
  \small
  \setlength{\tabcolsep}{3.5pt}
  \begin{tabular}{@{}l l cccccccc c@{}}
    \toprule
    Sparsity & Method & ARC-C & ARC-E & BoolQ & HSwag & MMLU & OBQA & RTE & Wino & \textbf{Avg} \\
    \midrule
    0\% & Full model & .625 & .838 & .887 & .797 & .802 & .446 & .769 & .736 & .737 \\
    \midrule
    \multirow{4}{*}{25\%} & EvoESAP (reported) & .583 & .794 & .883 & .733 & .733 & .416 & .754 & \textbf{.732} & .704 \\
    & EvoESAP (ours) & .583 & .796 & .880 & .730 & .729 & .408 & .751 & .722 & .700 \\
    & REAP uniform (ours) & \textbf{.586} & \textbf{.804} & .882 & .738 & \textbf{.734} & .408 & .769 & .723 & .706 \\
    & \textbf{RCO} & .584 & .797 & \textbf{.885} & \textbf{.757} & .733 & \textbf{.424} & \textbf{.783} & .714 & \textbf{.710} \\
    \midrule
    \multirow{2}{*}{50\%} & EvoESAP & \textbf{.437} & \textbf{.638} & .805 & .518 & .576 & .332 & .747 & .629 & .585 \\
    & \textbf{RCO} & .428 & .632 & \textbf{.835} & \textbf{.624} & \textbf{.587} & \textbf{.336} & \textbf{.751} & \textbf{.647} & \textbf{.605} \\
    \bottomrule
  \end{tabular}
\end{table}

At 25\% sparsity, RCO improves over our EvoESAP reproduction on seven of eight benchmarks, the gains being largest on RTE~(+3.2), HellaSwag~(+2.7), and OBQA~(+1.6); WinoGrande is the exception ($-$0.8).
Our reproduction agrees closely with the reported row (mean absolute per-benchmark difference 0.004, maximum 0.010), which corroborates those numbers independently.
Notably, uniform REAP allocation (.706) scores above EvoESAP's searched allocation (.700 ours, .704 reported), so on this metric the evolutionary search does not improve on not searching at all; RCO exceeds both.
Consistent with this, the searched allocation prunes 91\% of the same experts as uniform allocation, and its ESAP fitness moved by only $+0.0007$ over ten generations.
At 50\% sparsity, EvoESAP retains a small edge on ARC-C~($-$0.9) and ARC-E~($-$0.6), while RCO dominates on HellaSwag~(+10.6), BoolQ~(+3.0), and WinoGrande~(+1.8).

\subsection{Qwen3-Coder-Next: calibration and allocation ablation}
\label{app:coder_next_full}

Qwen3-Coder-Next has 512 routed experts per layer with 10 active per token.
We evaluate eight RCO variants spanning two calibration domains (coding: evol-codealpaca; general: FineWeb-Edu), two sparsity levels (25\%, 50\%), and two budget allocation strategies (uniform, nonuniform).
All evaluations use vLLM with bf16 and greedy decoding.

\paragraph{Uniform vs.\ nonuniform allocation.}
Each variant prunes a fixed fraction of total experts across all layers.
Uniform allocation keeps the same number of experts per layer (e.g., 384 at 25\%, 256 at 50\%).
Nonuniform allocation lets the optimizer freely distribute the pruning budget based on calibration loss: critical layers keep more experts, redundant layers are pruned more aggressively.
The effect is stark on coding benchmarks (Table~\ref{tab:coder_next_coding}): at 50\% sparsity, nonuniform recovers 97\% of HumanEval versus 55\% for uniform, a 42-point gap.
At 25\%, the gap narrows to 8 points (100\% vs.\ 92\%).
The gap grows with sparsity because uniform allocation forces equal pruning on layers that differ substantially in sensitivity; at low sparsity, even sensitive layers retain enough experts to function.

\paragraph{Coding vs.\ general calibration.}
RCO minimizes KL divergence on the calibration dataset, so the choice of calibration data determines which capabilities the pruned model preserves.
Coding-calibrated variants (evol-codealpaca) preserve code generation at the cost of general knowledge; general-calibrated variants (FineWeb-Edu) preserve Avg at the cost of coding ability.
The trade-off is sharp: general-calibrated variants lose coding ability almost entirely (HumanEval $\leq$6.1\% at 25\% sparsity), while coding-calibrated variants at 50\% sparsity still retain 76\% of Avg (Table~\ref{tab:coder_next_general}).
The asymmetry reflects the base model's specialization: its coding experts are concentrated in a small subset, easily lost when calibration does not exercise them.

\paragraph{Coding benchmarks.}
Table~\ref{tab:coder_next_coding} reports HumanEval (pass@1) and MBPP (pass@1) with recovery relative to the full model.

\begin{table}[H]
  \caption{Qwen3-Coder-Next coding benchmarks across calibration domain and allocation strategy.
  Recovery is relative to the full (unpruned) model.
  Bold: best among pruned models per sparsity level.
  All values in \%.}
  \label{tab:coder_next_coding}
  \centering
  \small
  \setlength{\tabcolsep}{4pt}
  \begin{tabular}{@{}l ll cc cc@{}}
    \toprule
    Sparsity & Cal.\ & Alloc.\ & HE & rec.\ & MBPP & rec.\ \\
    \midrule
    0\% & \multicolumn{2}{l}{Full model} & 74.4 & -- & 76.4 & -- \\
    \midrule
    \multirow{4}{*}{25\%} & coding & uniform & 68.3 & 92 & \textbf{68.8} & \textbf{90} \\
    & coding & nonunif.\ & \textbf{74.4} & \textbf{100} & 67.8 & 89 \\
    & general & uniform & 4.3 & 6 & 4.6 & 6 \\
    & general & nonunif.\ & 6.1 & 8 & 5.8 & 8 \\
    \midrule
    \multirow{4}{*}{50\%} & coding & uniform & 40.9 & 55 & 53.4 & 70 \\
    & coding & nonunif.\ & \textbf{72.0} & \textbf{97} & \textbf{69.0} & \textbf{90} \\
    & general & uniform & 0.0 & 0 & 1.8 & 2 \\
    & general & nonunif.\ & 1.2 & 2 & 1.0 & 1 \\
    \bottomrule
  \end{tabular}
\end{table}

\paragraph{General benchmarks.}
Table~\ref{tab:coder_next_general} reports all eight general benchmarks.
Avg is the unweighted mean.
General-calibrated variants preserve Avg up to 100\% at 25\% (nonuniform) and 92\% at 50\% (uniform).
At 50\% general sparsity, uniform (65.4) slightly outperforms nonuniform (64.4), in contrast to the coding results where nonuniform dominates.

\begin{table}[H]
  \caption{Qwen3-Coder-Next general benchmarks (Avg and per-benchmark breakdown).
  Bold: best among pruned models per sparsity level.
  All values in \%.}
  \label{tab:coder_next_general}
  \centering
  \small
  \setlength{\tabcolsep}{2.5pt}
  \begin{tabular}{@{}l ll cccccccc c@{}}
    \toprule
    Sparsity & Cal.\ & Alloc.\ & ARC-C & ARC-E & BoolQ & HSwag & MMLU & OBQA & RTE & Wino & Avg \\
    \midrule
    0\% & \multicolumn{2}{l}{Full model} & 60.6 & 82.1 & 88.5 & 77.5 & 76.7 & 43.0 & 76.5 & 66.6 & 71.4 \\
    \midrule
    \multirow{4}{*}{25\%}
    & coding & uniform & 50.1 & 72.2 & 86.4 & 69.0 & 71.0 & 38.0 & 72.9 & 65.5 & 65.6 \\
    & coding & nonunif.\ & 46.2 & 66.2 & 85.1 & 66.5 & 68.0 & 36.2 & 77.6 & 64.2 & 63.8 \\
    & general & uniform & 60.0 & 80.7 & 87.6 & \textbf{78.5} & 70.4 & \textbf{45.2} & 75.1 & 67.7 & 70.7 \\
    & general & nonunif.\ & \textbf{61.8} & \textbf{82.2} & \textbf{88.2} & 77.6 & \textbf{71.2} & 44.2 & \textbf{76.2} & \textbf{69.9} & \textbf{71.4} \\
    \midrule
    \multirow{4}{*}{50\%}
    & coding & uniform & 40.3 & 64.1 & 78.9 & 57.8 & 56.4 & 35.0 & 67.1 & 61.6 & 57.7 \\
    & coding & nonunif.\ & 35.6 & 55.5 & 77.6 & 54.8 & 54.3 & 34.0 & 64.6 & 60.3 & 54.6 \\
    & general & uniform & \textbf{54.1} & \textbf{77.1} & 83.9 & \textbf{70.9} & \textbf{61.0} & \textbf{42.8} & \textbf{67.5} & \textbf{65.8} & \textbf{65.4} \\
    & general & nonunif.\ & 52.6 & 76.2 & \textbf{84.2} & 70.8 & 59.5 & 41.4 & 67.5 & 63.5 & 64.4 \\
    \bottomrule
  \end{tabular}
\end{table}

\section{Full MoE Mixed-Precision Quantization Results}
\label{app:moe_quant_full}

\paragraph{Setup.}
MxMoE calibrates a per-(expert, weight-block, bitwidth) layer-output L2 sensitivity table, then solves a per-layer ILP that minimizes the sum of these sensitivities under a memory budget; the \texttt{gate} and \texttt{up} projections share a bitwidth and \texttt{down} is free, and attention is fixed at FP16 by design~\citep{duanmu2025mxmoe}.
RCO reads the same database and runs the manifold search to minimize calibration KL divergence under the same budget, with one bitwidth decision per expert (all three weights tied), a strictly smaller search space than MxMoE's per-weight allocation.
\textbf{RCO*} locks attention at FP16 to match MxMoE's design exactly, isolating the allocation algorithm; \textbf{RCO} additionally treats attention as searchable groups under a total-memory budget.
Reproduction details and accounting adjustments are in Appendix~\ref{app:mxmoe_details}.

\paragraph{Aggregate metrics.}
RCO* outperforms MxMoE on every aggregate metric at every target despite a smaller search space.
The Wikitext-2 PPL gap is $-0.43 / -0.18 / -0.04$ at $2.5 / 3.0 / 3.5$ bits, narrowing toward the FP16 floor (6.79) with more bits, and zero-shot accuracy favors RCO* by $0.3$--$0.6$ points.
RCO, which additionally quantizes attention to free expert budget at the same total memory, improves a further $0.01$--$0.08$ Wikitext-2 PPL over RCO* and adds another $0.0$--$0.4$ zero-shot points.

\paragraph{Saturation against uniform allocation.}
The uniform reference rows in Table~\ref{tab:moe_quant_main} bound where mixed precision matters: at 4 bit (uniform), the model is already within $0.7$ Wikitext-2 PPL and $0.7$ zero-shot points of FP16, so further bit allocation can only yield diminishing returns.
RCO at 3.5\,bit matches or beats uniform 4-bit on every metric (W2 6.88 vs 6.91, C4 10.23 vs 10.24, AvgPPL 8.76 vs 8.78, Avg ZS .683 vs .680) while using half a bit less per expert: the manifold has saturated against the post-quantization quality available from this GPTQ database.
At 2.5\,bit the picture is opposite: uniform 3-bit (W2 7.39, AvgPPL 9.38) is competitive with mixed precision in absolute terms but uses $0.5$ more bits per expert, so RCO at 2.5 bit recovers most of uniform 3-bit's quality at strictly lower memory.

\paragraph{Per-task zero-shot.}
On individual benchmarks, MxMoE retains an edge on a few (Winogrande at 3.0 and 3.5\,bit, ARC-Challenge at 2.5\,bit), but the gaps are within $0.4$--$2.1$ points; RCO and RCO* take the remaining benchmarks across all targets, with the largest gains on LAMBADA-OpenAI ($+1.4$ to $+1.6$ at 2.5\,bit) and ARC-Easy ($+2.9$ at 3.0\,bit).

\begin{table}[H]
  \caption{Full Qwen1.5-MoE-A2.7B mixed-precision quantization results, including the 3.0\,bit block omitted from Table~\ref{tab:moe_quant_main} for space.
  PPL ($\downarrow$) on Wikitext-2 (W2), C4, FineWeb-Edu (FW); zero-shot accuracy ($\uparrow$, decimals) on six lm-eval-harness tasks; Avg is the unweighted mean.}
  \label{tab:moe_quant_full}
  \centering
  \footnotesize
  \setlength{\tabcolsep}{3pt}
  \begin{tabular}{@{}l l cccc ccccccc@{}}
    \toprule
    & & \multicolumn{4}{c}{Perplexity ($\downarrow$)} & \multicolumn{7}{c}{Zero-shot accuracy ($\uparrow$)} \\
    \cmidrule(lr){3-6} \cmidrule(lr){7-13}
    Bits & Method & W2 & C4 & FW & Avg PPL & PIQA & HSwag & ARC-E & ARC-C & Wino & LAMB & Avg ZS \\
    \midrule
    FP16 & full model & 6.79 & 10.05 & 9.07 & 8.64 & .805 & .773 & .690 & .445 & .695 & .713 & .687 \\
    \midrule
    \multirow{3}{*}{2.5 bit}
    & MxMoE & 7.90 & 12.44 & 10.28 & 10.21 & .786 & .746 & .659 & \textbf{.416} & .653 & .662 & .653 \\
    & RCO* & 7.47 & 11.57 & 9.81 & 9.61 & \textbf{.793} & .750 & .652 & .411 & .651 & \textbf{.677} & .656 \\
    & RCO  & \textbf{7.40} & \textbf{11.44} & \textbf{9.74} & \textbf{9.53} & .792 & \textbf{.752} & \textbf{.667} & .409 & \textbf{.665} & .676 & \textbf{.660} \\
    \midrule
    \multirow{4}{*}{3.0 bit}
    & uniform & 7.39 & 11.05 & 9.71 & 9.38 & .789 & .761 & .650 & .436 & .663 & .689 & .664 \\
    & MxMoE & 7.21 & 10.90 & 9.51 & 9.21 & .796 & .764 & .655 & .434 & \textbf{.688} & .686 & .671 \\
    & RCO* & 7.03 & 10.54 & 9.32 & 8.96 & \textbf{.799} & .766 & \textbf{.684} & .442 & .673 & \textbf{.700} & \textbf{.677} \\
    & RCO  & \textbf{7.02} & \textbf{10.47} & \textbf{9.30} & \textbf{8.93} & .798 & \textbf{.767} & .677 & \textbf{.453} & .667 & \textbf{.700} & \textbf{.677} \\
    \midrule
    \multirow{3}{*}{3.5 bit}
    & MxMoE & 6.94 & 10.35 & 9.23 & 8.84 & \textbf{.801} & .768 & .659 & .433 & \textbf{.686} & .697 & .674 \\
    & RCO* & 6.90 & 10.25 & 9.18 & 8.77 & .799 & .770 & \textbf{.687} & .445 & .673 & .705 & .680 \\
    & RCO  & \textbf{6.88} & \textbf{10.23} & \textbf{9.17} & \textbf{8.76} & \textbf{.801} & \textbf{.771} & .686 & \textbf{.453} & .684 & \textbf{.706} & \textbf{.683} \\
    \midrule
    \multirow{1}{*}{4 bit}
    & uniform & 6.91 & 10.24 & 9.20 & 8.78 & .809 & .771 & .679 & .440 & .681 & .702 & .680 \\
    \bottomrule
  \end{tabular}
\end{table}

\section{Multiple Budgets and Penalty Comparison}
\label{app:multi_exp}

This section evaluates one instance of each kind of multiple budget from Section~\ref{sec:multi} and compares the manifold step against a quadratic penalty.
In every comparison only the manifold step (tangent projection, retraction, and momentum transport) is replaced; Gumbel sampling, the STE, calibration data, step count, and decoding are unchanged.
Every method's final assignment is decoded to the same budget, so KL differences reflect the continuous point the search reached rather than the deployed budget.

\subsection{Budgets on disjoint groups: equal expert counts per layer}
\label{app:multi_disjoint}

The uniform rows of Table~\ref{tab:coder_next_main} are a 48-budget problem: each of the 48 MoE layers of Qwen3-Coder-Next prunes exactly 128 (25\% sparsity) or 256 (50\%) of its 512 experts.
This is the allocation a pruned MoE needs for GGUF export, since the format stores a single expert count per model~\citep{gerganov2023llamacpp}.
Each budget acts on the experts of one layer, so $\mathbf{N}^\top\mathbf{N}$ and $\mathbf{J}$ are diagonal (Appendix~\ref{app:multi}) and each step runs 48 independent binary searches, together with per-layer projection, per-layer momentum transport, and a per-layer top-$k$ in the forward pass (Section~\ref{sec:algorithm}).
The global budget is implied by the per-layer budgets and is not imposed separately.
In all four uniform runs (coding and general calibration at 25\% and 50\% sparsity), the expected pruned count of every layer equals its target at every logged step (steps 0 to 149, logged every 10 steps).
A penalty formulation of the same problem requires one penalty term per layer; we did not run it.

\subsection{Budgets on shared groups: sparsity and bitwidth}
\label{app:multi_shared}

On OLMoE-1B-7B-0125-Instruct (16 layers $\times$ 64 routed experts), each of the 1024 experts chooses among \{prune, 2, 3, 4, 5 bits\} and each of the 16 attention groups among \{2, 3, 4, 5 bits\}.
Two equality budgets act on the same experts: an expected pruned fraction of 0.25 and an expected average bitwidth of 2.5 bits.
RCO uses the coupled projection and Newton retraction of Appendix~\ref{app:multi}.
The search runs 200 steps with $g{=}2$ Gumbel samples, batch size 1, and 64 FineWeb-Edu sequences of 2048 tokens (seed 42).
The baseline replaces the manifold step by the penalty $\lambda\bigl[(\E[\text{pruned fraction}] - 0.25)^2 + ((\E[\text{avg bits}] - 2.5)/2.5)^2\bigr]$ with one shared $\lambda \in \{0.1, 1, 10\}$.

\begin{table}[H]
  \caption{Two simultaneous budgets on OLMoE-1B-7B (targets: pruned fraction 0.25, average bitwidth 2.5). The expectations are the continuous state after 200 steps; KL ($\downarrow$) is measured on the decoded model, which meets both budgets for every method.}
  \label{tab:multi_budget}
  \centering
  \small
  \begin{tabular}{@{}l r cc cc@{}}
    \toprule
    Method & $\lambda$ & $\E[\text{pruned fraction}]$ & $\E[\text{avg bits}]$ & KL & KL / RCO \\
    \midrule
    RCO     & --  & \textbf{0.2500} & \textbf{2.500} & \textbf{0.910} & 1.00 \\
    Penalty & 0.1 & 0.0066 & 3.916 & 3.106 & 3.41 \\
    Penalty & 1   & 0.0181 & 3.361 & 2.257 & 2.48 \\
    Penalty & 10  & 0.1459 & 2.648 & 1.012 & 1.11 \\
    \bottomrule
  \end{tabular}
\end{table}

Table~\ref{tab:multi_budget} reports the results.
RCO holds both budgets within $10^{-4}$ throughout the search with no coefficient, the acceptance level of the single-precision LLM implementation (the MCKP experiments run in double precision to $10^{-8}$).
Every tested $\lambda$ misses both budgets, pruning 0.7--14.6\% of experts instead of 25\% at 2.65--3.92 expected bits, and reaches 1.11--3.41$\times$ RCO's KL.
At the best tested value, $\lambda{=}10$, the bitwidth budget is within 6\% while the pruned fraction is 42\% short, so a single coefficient does not balance the two budgets in the tested range.

\subsection{Single budgets}
\label{app:penalty_ablation}

\paragraph{Quantization.}
On Qwen3-8B at a target of 3.0 average bits, the search assigns bitwidths \{2, 3, 4, 5\} to 144 layer groups (36 layers $\times$ 4: q/k/v, o, gate/up, down), running 300 steps with $g{=}4$ on 64 FineWeb-Edu sequences of 4096 tokens.
The penalty $\lambda\,(C(\boldsymbol{\alpha}) - B)^2$ replaces the manifold step, either with the DP kept in the forward pass and at decoding, or with the DP removed (plain argmax), so that only the penalty enforces the budget.
Table~\ref{tab:penalty_quant} reports both variants.
With the DP, every run decodes to 3.001 bits and a tuned coefficient comes within 1.3\% of RCO's KL (0.1374 vs.\ 0.1356 at $\lambda{=}100$), while small coefficients let the continuous iterate drift up to 0.84 bits from the budget.
Without the DP, no tested $\lambda$ deploys a model on budget (deviations 0.007--0.531 bits); $\lambda{=}0.1$ and $\lambda{=}1$ deploy models above 3.0 bits, so their lower KL reflects the extra bits and is not comparable to RCO's.
For a single quantization budget the gain from the manifold is therefore exact feasibility without tuning, not solution quality.

\begin{table}[H]
  \caption{Single quantization budget on Qwen3-8B (target 3.0 bits). Max drift: largest $|\E[\text{bits}] - 3|$ of the continuous iterate during the search (not recorded for $\lambda{=}100$). KL / RCO is omitted for runs deployed above the budget.}
  \label{tab:penalty_quant}
  \centering
  \small
  \begin{tabular}{@{}l l r ccc c@{}}
    \toprule
    Method & Decoding & $\lambda$ & Deployed bits & Max drift & KL & KL / RCO \\
    \midrule
    RCO     & DP     & --  & 3.001 & 0.0001 & 0.1356 & 1.00 \\
    \midrule
    Penalty & DP     & 0.1 & 3.001 & 0.84 & 0.3335 & 2.46 \\
    Penalty & DP     & 1   & 3.001 & 0.25 & 0.1826 & 1.35 \\
    Penalty & DP     & 10  & 3.001 & 0.09 & 0.1396 & 1.03 \\
    Penalty & DP     & 100 & 3.001 & --   & 0.1374 & 1.01 \\
    \midrule
    Penalty & argmax & 0.1 & 3.531 & 0.64 & 0.0664 & -- \\
    Penalty & argmax & 1   & 3.077 & 0.24 & 0.1274 & -- \\
    Penalty & argmax & 10  & 3.007 & 0.09 & 0.1373 & 1.01 \\
    Penalty & argmax & 100 & 2.983 & --   & 0.1381 & 1.02 \\
    \bottomrule
  \end{tabular}
\end{table}

\paragraph{Expert pruning.}
On OLMoE-1B-7B-0125-Instruct, the search prunes 256 of 1024 experts (25\%) in 200 steps with $g{=}2$, batch size 1, and 64 FineWeb-Edu sequences of 2048 tokens (seed 42).
The penalty $\lambda\,(\E[\text{number pruned}] - 256)^2$ replaces the manifold step, and every method's final mask prunes exactly 256 experts.
Table~\ref{tab:penalty_prune} counts an expert as committed when its keep or prune probability exceeds 0.9 after 200 steps.
Under the penalty at most 5 of 1024 experts commit (RCO: 879), and the decoded models reach 1.75--1.96$\times$ RCO's KL.
The penalty runs' expected pruned count stays near 256 only because every expert's pruning probability stays near 0.25, so the decoded mask is one the search never trained under.

\begin{table}[H]
  \caption{Single expert-pruning budget on OLMoE-1B-7B (256 of 1024 experts). Committed: experts with keep or prune probability above 0.9 after 200 steps. KL ($\downarrow$) on the decoded model.}
  \label{tab:penalty_prune}
  \centering
  \small
  \begin{tabular}{@{}l r ccc@{}}
    \toprule
    Method & $\lambda$ & Committed & KL & KL / RCO \\
    \midrule
    RCO     & --  & \textbf{879} & \textbf{0.4265} & 1.00 \\
    Penalty & 0.1 & 5   & 0.7447 & 1.75 \\
    Penalty & 1   & 0   & 0.8348 & 1.96 \\
    Penalty & 10  & 0   & 0.8094 & 1.90 \\
    \bottomrule
  \end{tabular}
\end{table}

\section{Mixed-Precision Quantization: Ablation Study}
\label{app:quant_ablations}

We report ablation sweeps for RCO applied to mixed-precision quantization on Qwen3-8B (252 linear layers, 7 bitwidth options: 2, 3, 4, 5, 6, 7, 8) at a target of 2.5 average bits per parameter.
All layers are treated as independent groups (no structural grouping unless stated otherwise).
Calibration uses FineWeb-Edu with 256 samples at sequence length 2048 unless noted.
Evaluation reports perplexity on two held-out corpora: FineWeb-Edu (FW) and C4, each with 131k tokens at sequence length 2048.
The FP16 baseline achieves FW=10.96, C4=17.20.

\paragraph{Hardware.}
All ablation sweeps use an RTX 3090 (24\,GB).
The baseline comparison (Table~\ref{tab:quant_main}) reports wall times from an RTX A6000 (48\,GB).

\subsection{Gumbel samples per step}
\label{app:quant_gumbel}

Each optimization step draws $g$ independent Gumbel-softmax assignments and averages their STE gradients before the optimizer update.
More samples reduce gradient variance over the discrete assignment space at the cost of additional forward-backward passes per step.

\begin{table}[H]
  \caption{Effect of Gumbel sample count ($g$).
  Fixed: $T{=}200$, $\tau_{\min}{=}0.01$, lr$=$0.1, seed=42, cal=256.
  Hardware: RTX 3090.}
  \label{tab:quant_gumbel}
  \centering
  \small
  \begin{tabular}{@{}r cc r r@{}}
    \toprule
    $g$ & FW$\downarrow$ & C4$\downarrow$ & Wall & Fwd passes \\
    \midrule
    1  & 16.46 & 25.63 & 23\,m & 200 \\
    4  & 15.61 & 24.29 & 62\,m & 800 \\
    8  & 15.66 & 24.13 & 114\,m & 1600 \\
    16 & 15.51 & 24.30 & 218\,m & 3200 \\
    \bottomrule
  \end{tabular}
\end{table}

The largest gain occurs from $g{=}1$ to $g{=}4$ ($-$0.85 FW perplexity). Further increases to $g{=}16$ yield diminishing and non-monotonic gains on FW (within seed noise), though C4 improves monotonically.
Gradient variance over discrete assignments is the primary bottleneck: more samples provide cleaner per-layer signal about which bitwidth to prefer.

\subsection{Optimization steps}
\label{app:quant_steps}

Temperature anneals exponentially from $\tau_0{=}1.0$ to $\tau_{\min}$ over $T$ steps, so more steps produce slower annealing.

\begin{table}[H]
  \caption{Effect of optimization steps ($T$).
  Fixed: $g{=}4$, $\tau_{\min}{=}0.01$, lr$=$0.1, seed=42, cal=256.
  Hardware: RTX 3090.}
  \label{tab:quant_steps}
  \centering
  \small
  \begin{tabular}{@{}r cc r r@{}}
    \toprule
    $T$ & FW$\downarrow$ & C4$\downarrow$ & Wall & Fwd passes \\
    \midrule
    100 & 15.62 & 24.17 & 35\,m & 400 \\
    200 & 15.61 & 24.29 & 62\,m & 800 \\
    400 & 15.68 & 24.34 & 115\,m & 1600 \\
    \bottomrule
  \end{tabular}
\end{table}

Results are non-monotonic: $T{=}100$ and $T{=}200$ are nearly identical, while $T{=}400$ is slightly worse.
The optimization converges quickly; additional steps at low temperature accumulate biased STE gradients without improving the solution.

\subsection{Samples vs.\ steps at fixed compute}
\label{app:quant_compute_matched}

For a fixed compute budget (same number of forward passes), is it better to increase Gumbel samples per step or to run more steps?

\begin{table}[H]
  \caption{Compute-matched configurations.
  $g{=}32$, $T{=}50$ achieves the same quality as $g{=}16$, $T{=}200$ at half the compute (1600 vs.\ 3200 forward passes).
  Note: $g{=}32$ uses seed=2; others use seed=42.
  Hardware: RTX 3090.}
  \label{tab:quant_compute}
  \centering
  \small
  \begin{tabular}{@{}rr cc r r@{}}
    \toprule
    $g$ & $T$ & FW$\downarrow$ & C4$\downarrow$ & Fwd passes & Wall \\
    \midrule
    4  & 200 & 15.61 & 24.29 & 800 & 62\,m \\
    8  & 100 & 15.68 & 24.28 & 800 & 61\,m \\
    16 & 200 & 15.53 & 24.13 & 3200 & 202\,m \\
    32 & 50  & 15.56 & 24.10 & 1600 & 110\,m \\
    \bottomrule
  \end{tabular}
\end{table}

Sample efficiency dominates step count.
The STE gradient benefits more from averaging over diverse discrete assignments than from additional optimization iterations at low temperature.

\subsection{Temperature minimum}
\label{app:quant_tau}

Temperature decays exponentially from $\tau_0{=}1.0$ to $\tau_{\min}$.
Lower $\tau_{\min}$ forces sharper assignments earlier; higher $\tau_{\min}$ leaves more uncertainty at convergence.
The final discrete assignment is extracted via budget-constrained argmax regardless of $\tau_{\min}$.

\begin{table}[H]
  \caption{Effect of $\tau_{\min}$.
  Fixed: $g{=}4$, $T{=}200$, lr$=$0.1, seed=42, cal=256.
  Hardware: RTX 3090.}
  \label{tab:quant_tau}
  \centering
  \small
  \begin{tabular}{@{}r cc@{}}
    \toprule
    $\tau_{\min}$ & FW$\downarrow$ & C4$\downarrow$ \\
    \midrule
    0.05  & 15.75 & 24.55 \\
    \textbf{0.01}  & \textbf{15.61} & \textbf{24.29} \\
    0.005 & 15.77 & 24.42 \\
    0.001 & 16.26 & 25.20 \\
    \bottomrule
  \end{tabular}
\end{table}

$\tau_{\min}{=}0.01$ is optimal.
Higher values leave too many layers with soft, undecided assignments.
Lower values force premature hard decisions before layers have accumulated sufficient gradient signal, locking in suboptimal assignments early.

\subsection{Learning rate}
\label{app:quant_lr}

The Adam optimizer step size in logit space.

\begin{table}[H]
  \caption{Effect of learning rate.
  Fixed: $g{=}4$, $T{=}200$, $\tau_{\min}{=}0.01$, seed=42, cal=256.
  Hardware: RTX 3090.}
  \label{tab:quant_lr}
  \centering
  \small
  \begin{tabular}{@{}r cc@{}}
    \toprule
    LR & FW$\downarrow$ & C4$\downarrow$ \\
    \midrule
    0.05 & 15.66 & 24.33 \\
    0.1  & 15.61 & 24.29 \\
    0.2  & 15.63 & 24.40 \\
    \bottomrule
  \end{tabular}
\end{table}

Results are robust across a 4$\times$ range (within seed noise of $\sim$0.2 PPL).
The manifold projection removes the budget-changing gradient component, and the retraction prevents constraint violation, leaving the optimizer in a well-conditioned subspace.

\subsection{Antithetic sampling}
\label{app:quant_antithetic}

Antithetic sampling pairs each uniform draw $U$ with $1{-}U$, producing two negatively correlated Gumbel variates per draw.
This halves the number of independent assignments explored per step.

\begin{table}[H]
  \caption{Effect of antithetic sampling.
  Fixed: $g{=}4$, $T{=}200$, $\tau_{\min}{=}0.01$, lr$=$0.1, seed=42, cal=256.
  Hardware: RTX 3090.}
  \label{tab:quant_antithetic}
  \centering
  \small
  \begin{tabular}{@{}l cc@{}}
    \toprule
    Antithetic & FW$\downarrow$ & C4$\downarrow$ \\
    \midrule
    No  & \textbf{15.61} & \textbf{24.29} \\
    Yes & 15.97 & 24.63 \\
    \bottomrule
  \end{tabular}
\end{table}

Antithetic sampling hurts.
With $g{=}4$, the two base draws plus two mirrored draws reduce diversity in the discrete assignment space.
The variance reduction from negative correlation does not compensate for the reduced exploration.
Independent samples explore the combinatorial space more effectively.

\subsection{Seed variance}
\label{app:quant_seed}

The random seed affects calibration data ordering, Gumbel noise sequences, and optimizer initialization.

\begin{table}[H]
  \caption{Seed variance.
  Fixed: $g{=}4$, $T{=}200$, $\tau_{\min}{=}0.01$, lr$=$0.1, cal=256.
  Hardware: RTX 3090.}
  \label{tab:quant_seed}
  \centering
  \small
  \begin{tabular}{@{}r cc@{}}
    \toprule
    Seed & FW$\downarrow$ & C4$\downarrow$ \\
    \midrule
    2  & 15.57 & 24.24 \\
    42 & 15.61 & 24.29 \\
    0  & 15.83 & 24.84 \\
    1  & 15.96 & 25.23 \\
    3  & 15.98 & 24.71 \\
    \bottomrule
  \end{tabular}
\end{table}

FW perplexity: mean=15.79, std=0.18, range=0.41.
The range exceeds most hyperparameter effects except Gumbel sample count.
Approximately 180 of 252 layers converge to the same assignment across seeds; the remaining $\sim$70 ambiguous layers (where multiple bitwidths yield similar KL) account for the variance.
More Gumbel samples reduce this variance by providing better gradient signal for ambiguous layers.

\subsection{Layer grouping}
\label{app:quant_grouping}

Structural grouping forces layers with similar roles (e.g., gate\_proj and up\_proj, or q/k/v\_proj) to share a bitwidth, reducing the search space from 252 to $\sim$144 groups.

\begin{table}[H]
  \caption{Effect of layer grouping.
  Fixed: $g{=}4$, $T{=}200$, $\tau_{\min}{=}0.01$, lr$=$0.1, seed=42.
  Hardware: RTX 3090.}
  \label{tab:quant_grouping}
  \centering
  \small
  \begin{tabular}{@{}l r cc@{}}
    \toprule
    Grouping & Cal & FW$\downarrow$ & C4$\downarrow$ \\
    \midrule
    None    & 256  & \textbf{15.60} & \textbf{24.14} \\
    Grouped & 256  & 16.11 & 25.36 \\
    None    & 512  & \textbf{15.82} & \textbf{24.83} \\
    Grouped & 512  & 16.94 & 26.98 \\
    None    & 1024 & \textbf{16.14} & \textbf{25.13} \\
    Grouped & 1024 & 16.40 & 25.73 \\
    \bottomrule
  \end{tabular}
\end{table}

No grouping consistently outperforms grouped assignment by 0.3--1.1 FW perplexity.
Individual layers within structural groups can have different quantization sensitivity, so forcing a shared bitwidth is suboptimal.
RCO scales well to the full 252-dimensional search space without grouping.

\subsection{Calibration samples}
\label{app:quant_cal}

More calibration sequences produce a less noisy KL objective but require proportionally more RAM for cached reference log-probabilities ($\sim$159\,GB for 256 samples).

\begin{table}[H]
  \caption{Effect of calibration set size.
  Note: cal=512 runs use lr$=$0.2; comparison is not perfectly controlled.
  Hardware: RTX 3090.}
  \label{tab:quant_cal}
  \centering
  \small
  \begin{tabular}{@{}r l cc@{}}
    \toprule
    Cal & Config & FW$\downarrow$ & C4$\downarrow$ \\
    \midrule
    256 & $g{=}8$, $T{=}200$, lr=0.1 & 15.78 & 24.48 \\
    512 & $g{=}8$, $T{=}200$, lr=0.2 & 15.65 & 24.51 \\
    256 & $g{=}4$, $T{=}200$, lr=0.1 & 15.81 & 24.73 \\
    512 & $g{=}4$, $T{=}200$, lr=0.2 & 16.23 & 25.38 \\
    \bottomrule
  \end{tabular}
\end{table}

No clear benefit from doubling calibration data.
The dominant source of variance is the discrete assignment landscape (seed variance, Section~\ref{app:quant_seed}), not calibration noise.
At 256 samples (524k tokens), the gradient signal is sufficient for the $\sim$252 decisions.

\subsection{Scaling: Gumbel samples with fixed steps}
\label{app:quant_scaling}

Table~\ref{tab:quant_scaling} summarizes the compute-quality tradeoff as Gumbel sample count increases, alongside EvoPress.

\begin{table}[H]
  \caption{Scaling behavior.
  RCO with $g{=}32$, $T{=}50$ matches or exceeds EvoPress on both FW and C4.
  The compute-optimal strategy maximizes samples per step while minimizing step count.
  RCO rows on RTX 3090; EvoPress row from the matched-protocol reproduction in Table~\ref{tab:quant_main} (RTX A6000).
}
  \label{tab:quant_scaling}
  \centering
  \small
  \begin{tabular}{@{}l rr cc r@{}}
    \toprule
    Method & $g$ & $T$ & FW$\downarrow$ & C4$\downarrow$ & Wall \\
    \midrule
    RCO         & 1  & 200 & 16.46 & 25.63 & 23\,m \\
    RCO         & 4  & 200 & 15.61 & 24.29 & 62\,m \\
    RCO         & 8  & 200 & 15.66 & 24.13 & 114\,m \\
    RCO         & 16 & 200 & \textbf{15.53} & 24.30 & 218\,m \\
    \textbf{RCO} & \textbf{32} & \textbf{50} & 15.56 & \textbf{24.10} & \textbf{110\,m} \\
    \midrule
    EvoPress    & = & 100 gen & 15.64 & 24.63 & 11--14\,h \\
    \bottomrule
  \end{tabular}
\end{table}

\subsection{Recommended defaults}
\label{app:quant_defaults}

Table~\ref{tab:quant_defaults} summarizes the recommended hyperparameters based on the ablations above.

\begin{table}[H]
  \caption{Recommended RCO hyperparameters for mixed-precision quantization.}
  \label{tab:quant_defaults}
  \centering
  \small
  \begin{tabular}{@{}l c c l@{}}
    \toprule
    Parameter & Recommended & Robust range & Note \\
    \midrule
    Gumbel samples ($g$) & 16--32 & 4--64 & More is better; diminishing returns \\
    Steps ($T$) & 50--200 & 50--400 & Non-monotonic; fewer is fine with more samples \\
    $\tau_{\min}$ & 0.01 & 0.005--0.05 & Too low: premature decisions \\
    Learning rate & 0.1 & 0.05--0.2 & Robust due to manifold projection \\
    Antithetic & No & - & Reduces assignment diversity \\
    Calibration & 256 & 128--512 & Not the bottleneck \\
    Layer grouping & None & - & Layers differ in sensitivity \\
    \bottomrule
  \end{tabular}
\end{table}

\section{Baseline Reproduction Details}
\label{app:baseline_details}

The EvoPress baseline is our own reproduction under a matched evaluation protocol; the EvoESAP numbers are taken from \citet{liu2026evoesap}, except at Qwen3-30B-A3B 25\% sparsity, where we report both their numbers and our own reproduction.
This section documents the setup for each.

\subsection{EvoESAP}
\label{app:evoesap_details}

All EvoESAP numbers (downstream accuracies and search times) are taken from \citet{liu2026evoesap}, for the same model checkpoints (OLMoE-1B-7B-0125-Instruct and Qwen3-30B-A3B-Instruct-2507).
We report the variant with REAP as the within-layer importance criterion since it is the strongest of the four criteria evaluated by \citet{liu2026evoesap}.
Their search times are for 50 generations on one NVIDIA L40S GPU (OLMoE-1B-7B) and 10 generations on two L40S GPUs (Qwen3-30B-A3B).
These numbers were obtained under their protocol (REAP calibrated on evol-codealpaca-v1, search on tulu-3-sft-personas-math with the ESAP fitness) and evaluation pipeline, not ours; the Qwen3-30B-A3B full-model row in Table~\ref{tab:qwen3_full} is likewise taken from \citet{liu2026evoesap}.

\paragraph{Our reproduction (Qwen3-30B-A3B, 25\%).}
At the 25\% point we also run the EvoESAP search ourselves from the authors' released code\footnote{\url{https://github.com/ZongfangLiu/EvoESAP}} in their published configuration: REAP saliency calibrated on 1024 evol-codealpaca-v1 samples; ESAP fitness (\texttt{esp-dataset}) on 64 tulu-3-sft-personas-math samples; population 32, 10 generations, top-4 survivors, mutation $4{\times}3$, seed 42; budget 32 pruned experts per layer. The search returns a non-uniform allocation (21--43 experts pruned per layer) that, within each layer, prunes exactly the lowest-saliency experts under REAP, so it varies only the per-layer counts. We evaluate this allocation under the protocol used for RCO and report it as ``EvoESAP (ours)'' in Table~\ref{tab:qwen3_full}; ``REAP uniform (ours)'' prunes 32 experts in every layer under the same saliency. The reproduction took 4.3\,h for calibration and 5.9\,h for the search on one H200; \citet{liu2026evoesap} report search time only (5.2\,h on two L40S), though calibration is the costlier stage in their implementation.

Two deviations were required. We set \texttt{record\_pruning\_metrics\_only} to skip the expert-\emph{merging} statistics that REAP pruning does not read, which reduces calibration from 29\,h to 4.3\,h; the resulting REAP saliency is bit-identical to the full computation (maximum absolute difference 0.0 across all 48 layers, identical pruned-expert set per layer). We also replace the authors' mask-based evaluation path, which assumes the experts form an \texttt{nn.ModuleList} and fails silently on the fused expert module of recent \texttt{transformers} releases that their code otherwise requires: we install the router masks with their own search-time masking function and confirm that no pruned expert is routed to.

\subsection{EvoPress}
\label{app:evopress_details}

All EvoPress numbers for Qwen3-8B are our own reproductions, run using the authors' released code\footnote{\url{https://github.com/IST-DASLab/EvoPress}} under our evaluation pipeline (FineWeb-Edu calibration, KL divergence fitness, 256 calibration samples, same eval setup as RCO).
\citet{sieberling2024evopress} do not include Qwen3-8B in their reported experiments.
We run EvoPress for 100 generations.
RCO and EvoPress are measured in the same compute environment.

\subsection{MxMoE}
\label{app:mxmoe_details}

All MxMoE numbers for Qwen1.5-MoE-A2.7B are produced from the authors' released code\footnote{\url{https://github.com/cat538/MxMoE}} under our evaluation pipeline: a shared GPTQ weight database (bitwidths 2--8, group size 128, asymmetric, no Hadamard) read by both methods, MxMoE's per-bitwidth sensitivity calibration command (\texttt{mxmoe.quant.quant calib --metric layer\_out\_norm}, FineWeb-Edu, 64 samples $\times$ 2048 tokens; theirs: WikiText-2, 128 $\times$ 4096), the per-layer ILP allocator (binary $x_{e,n,s}$, objective $\min \sum_{e,n,s} x_{e,n,s} \cdot \delta_{e,n,s}$ on the layer-output sensitivity $\delta$, with \texttt{gate}$\equiv$\texttt{up} coupling and \texttt{down} free; attention is left at FP16 since their pipeline hardcodes \texttt{w\_bits=16} for it), the Wikitext-2 PPL protocol (100 chunks of 4096 tokens), the zero-shot task list (PIQA, HellaSwag, ARC-Easy, ARC-Challenge, Winogrande, LAMBADA-OpenAI), and our harness for both PPL and zero-shot. Our only other modification is to remove the $+0.25$-bit group-128 metadata overhead from the bit accounting so that both methods report the same bit rate. We omit three MxMoE features that are orthogonal to the allocation algorithm: the online Hadamard transform (\texttt{gptq-had}, their flagship variant; the weight database is our own and would need to be rebuilt with Hadamard, though Hadamard would benefit RCO equally), joint accuracy and throughput optimization (does not affect the accuracy results we report), and weight-and-activation quantization (W4A4, W8A8; our comparison is weight-only).

\section{Full MCKP Results}
\label{app:mckp_full}

We evaluate on 13 scenarios (3 random instances per scenario, 5000 steps, learning rate 0.01) spanning correlated costs, tight budgets, adversarial instances, under-budget optima, and a large-scale instance ($N{=}1000$, $K{=}32$).
Table~\ref{tab:mckp_full} reports the five scenarios on which the methods diverge; the remaining eight place all four methods (equality, slack, Lagrangian, augmented Lagrangian) within instance-to-instance noise of each other.
Gap denotes the percentage below the DP optimum for the greedy-repaired discrete assignment; violation is the mean absolute constraint violation $|C(\boldsymbol{\alpha}) - B|$ averaged over all optimization steps.

\begin{table}[H]
  \caption{MCKP benchmark results on the scenarios that distinguish methods (5000 steps, 3 random instances per scenario).
  On the remaining eight scenarios all methods land within instance noise of each other.
  Gap: \% below DP optimum on the greedy-repaired discrete assignment (mean $\pm$ std over instances).
  Violation: mean $|E[\text{cost}] - B|$ averaged over all steps.
  Budget manifold maintains constraint violation $<10^{-8}$ throughout.
  $^\dagger$The slack variant of Budget manifold recovers the DP optimum on both \textit{cheap optimal} and \textit{mixed slack} (gap $<0.01\%$), while equality-constrained methods waste the surplus budget on inferior options.}
  \label{tab:mckp_full}
  \centering
  \small
  \begin{tabular}{@{}l r@{$\,\pm\,$}l r r@{$\,\pm\,$}l r r@{$\,\pm\,$}l r@{}}
    \toprule
    & \multicolumn{3}{c}{Budget manifold} & \multicolumn{3}{c}{Lagrangian} & \multicolumn{3}{c}{Aug.\ Lagrangian} \\
    \cmidrule(lr){2-4} \cmidrule(lr){5-7} \cmidrule(lr){8-10}
    Scenario & \multicolumn{2}{c}{Gap\%} & $|\text{v}|$ & \multicolumn{2}{c}{Gap\%} & $|\text{v}|$ & \multicolumn{2}{c}{Gap\%} & $|\text{v}|$ \\
    \midrule
    \multicolumn{10}{@{}l}{\textit{Scale: Lagrangian penalty degrades with $N$, $K$.}} \\
    \textit{large} ($N{=}200$, $K{=}16$)         & \textbf{0.02} & 0.02 & $<10^{-8}$ & 0.02  & 0.02 & 0.08 & 0.53 & 0.12 & 0.03 \\
    \textit{float costs} ($N{=}200$, $K{=}16$)   & \textbf{0.02} & 0.02 & $<10^{-8}$ & 0.02  & 0.02 & 0.08 & 0.64 & 0.04 & 0.03 \\
    \textit{huge} ($N{=}1000$, $K{=}32$)         & \textbf{0.00} & 0.00 & $<10^{-8}$ & 37.56 & 1.78 & 1.15 & 4.28 & 0.20 & 0.12 \\
    \midrule
    \multicolumn{10}{@{}l}{\textit{Under-budget optima: equality methods waste budget; the slack variant$^\dagger$ recovers the optimum.}} \\
    \textit{cheap optimal}                       & 48.71 & 1.41 & $<10^{-8}$ & 48.75 & 1.14 & 0.09 & 48.75 & 1.14 & 0.01 \\
    \textit{mixed slack}                         & 12.57 & 1.66 & $<10^{-8}$ & 12.35 & 1.65 & 0.08 & 12.35 & 1.48 & 0.00 \\
    \bottomrule
  \end{tabular}
\end{table}

The results fall into two regimes that distinguish methods, plus a residual.

\paragraph{Scale} (\textit{large}, \textit{float costs}, \textit{huge}).
Lagrangian penalty performance degrades sharply with $N$ and $K$: at $N{=}1000$ ($K{=}32$), Lagrangian plateaus at a 37.6\% gap while Budget manifold reaches the DP optimum exactly.
Augmented Lagrangian is more robust at scale but still trails Budget manifold by an order of magnitude on the largest instance and by 25--30$\times$ on the moderate-scale ones.
Across all three scenarios, Budget manifold's repaired gap stays at $\leq 0.02\%$.

\paragraph{Under-budget optima} (\textit{cheap optimal}, \textit{mixed slack}).
When the value-maximizing assignment costs less than the budget, equality-constrained methods are forced to spend the remaining budget on inferior options, losing up to half the optimal value.
The slack variable resolves this completely: Budget manifold (inequality) recovers the DP optimum exactly.
This scenario is relevant in practice when budget targets are set conservatively.

\paragraph{Easy / tied scenarios.}
On \textit{small easy}, \textit{medium}, \textit{tight budget}, \textit{adversarial}, and \textit{nonuniform}, the four methods reach comparable final gaps and comparable compute (within instance-to-instance noise of each other).
Constraint feasibility nonetheless separates them throughout training: Budget manifold (eq.) maintains $|C(\boldsymbol{\alpha}) - B| < 10^{-8}$ in every scenario, whereas Lagrangian methods carry a mean violation of order $10^{-1}$ (Figures~\ref{fig:mckp_adv_corr}, \ref{fig:mckp_bound_cheap}).

\paragraph{Compute-matched comparison across scenarios.}
Final-step gap masks substantial differences in convergence speed.
Table~\ref{tab:mckp_compute_full} reports the number of steps each method needs to reach within 1\% of the DP optimum on the eight scenarios that distinguish them.
Three regimes emerge.
At scale ($N \geq 200$), augmented Lagrangian does not converge to 1\% within the 5000-step budget on any of \textit{large}, \textit{float costs}, or \textit{huge}; vanilla Lagrangian additionally plateaus indefinitely on \textit{huge}.
On hard small-scale instances (\textit{boundary}, \textit{correlated}, \textit{correlated tight}), the manifold reaches the threshold 2--5$\times$ faster than augmented Lagrangian and 1.5--2.5$\times$ faster than vanilla Lagrangian.
On the under-budget scenarios (\textit{cheap optimal}, \textit{mixed slack}), only the slack variant converges; the equality manifold and both Lagrangian baselines plateau indefinitely on the under-budget portion of the surface.

\begin{table}[H]
  \caption{Compute (in optimization steps) for each method to reach within 1\% of the DP optimum, on the eight scenarios that distinguish methods (averaged over 3 random instances per scenario; \textit{never} indicates the threshold is not reached within 5000 steps).
  Bold marks the fastest method in each row.
  $^\dagger$Under-budget scenarios: equality methods structurally cannot leave the budget surface; the reported value is the slack variant, which converges in $\sim$500 steps.}
  \label{tab:mckp_compute_full}
  \centering
  \small
  \begin{tabular}{@{}l rrr@{}}
    \toprule
    Scenario & Manifold & Lagrangian & Aug.\ Lagrangian \\
    \midrule
    \multicolumn{4}{@{}l}{\textit{Scale: $N \geq 200$.}} \\
    \textit{huge} ($N{=}1000$, $K{=}32$)         & \textbf{594} & \textit{never} & \textit{never} \\
    \textit{large} ($N{=}200$, $K{=}16$)         & \textbf{583} & \textbf{583}   & \textit{never} \\
    \textit{float costs} ($N{=}200$, $K{=}16$)   & \textbf{583} & 585            & \textit{never} \\
    \midrule
    \multicolumn{4}{@{}l}{\textit{Hard small-scale: correlated values or boundary-clustered structure.}} \\
    \textit{boundary} ($N{=}50$, $K{=}8$)         & \textbf{185} & 453            & 900  \\
    \textit{correlated} ($N{=}50$, $K{=}8$)       & \textbf{381} & 581            & 1444 \\
    \textit{correlated tight} ($N{=}50$, $K{=}8$) & \textbf{308} & 482            & 1029 \\
    \midrule
    \multicolumn{4}{@{}l}{\textit{Under-budget optima: only the slack variant of the manifold converges.}} \\
    \textit{cheap optimal} ($N{=}50$, $K{=}8$)    & \textbf{562}$^\dagger$ & \textit{never} & \textit{never} \\
    \textit{mixed slack} ($N{=}50$, $K{=}8$)      & \textbf{514}$^\dagger$ & \textit{never} & \textit{never} \\
    \bottomrule
  \end{tabular}
\end{table}

Figures~\ref{fig:mckp_adv_corr} and~\ref{fig:mckp_bound_cheap} show per-instance convergence traces for representative scenarios.
The raw constraint violation traces (panel~c in each figure) make the oscillation of Lagrangian methods visible: the penalty weight overshoots and undershoots the budget repeatedly, while Budget manifold maintains violation at $\sim\!10^{-8}$ throughout.

\begin{figure}[H]
  \centering
  \begin{subfigure}[t]{\linewidth}
    \centering
    \includegraphics[width=\linewidth]{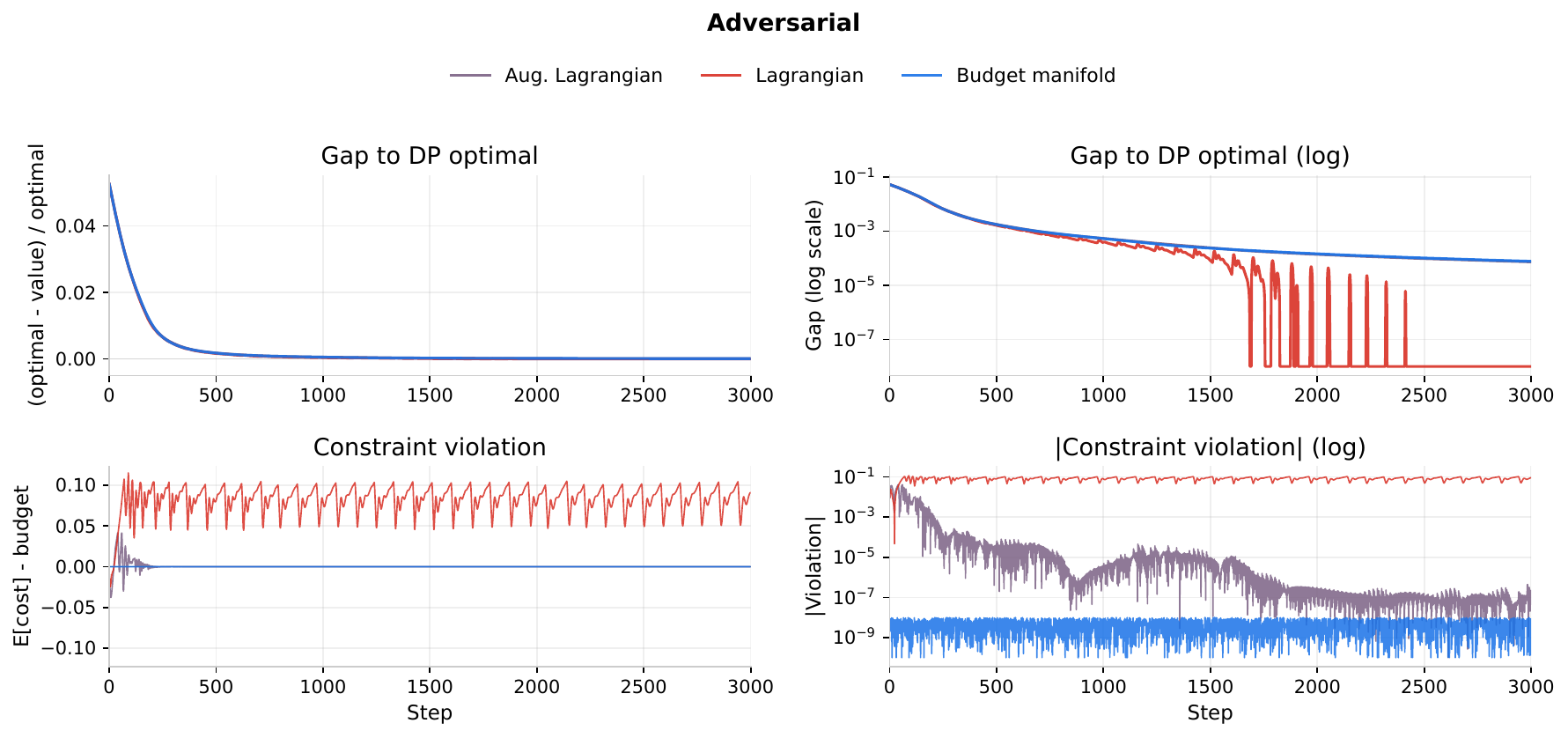}
    \caption{Adversarial scenario ($N{=}50$, $K{=}8$, single instance).
    Costs clustered near budget$/N$, values nearly tied.
    The Lagrangian penalty oscillates persistently (panel~c), never satisfying the constraint.
    Budget manifold converges to the DP optimum with zero violation.}
    \label{fig:mckp_adv}
  \end{subfigure}

  \vspace{3em}

  \begin{subfigure}[t]{\linewidth}
    \centering
    \includegraphics[width=\linewidth]{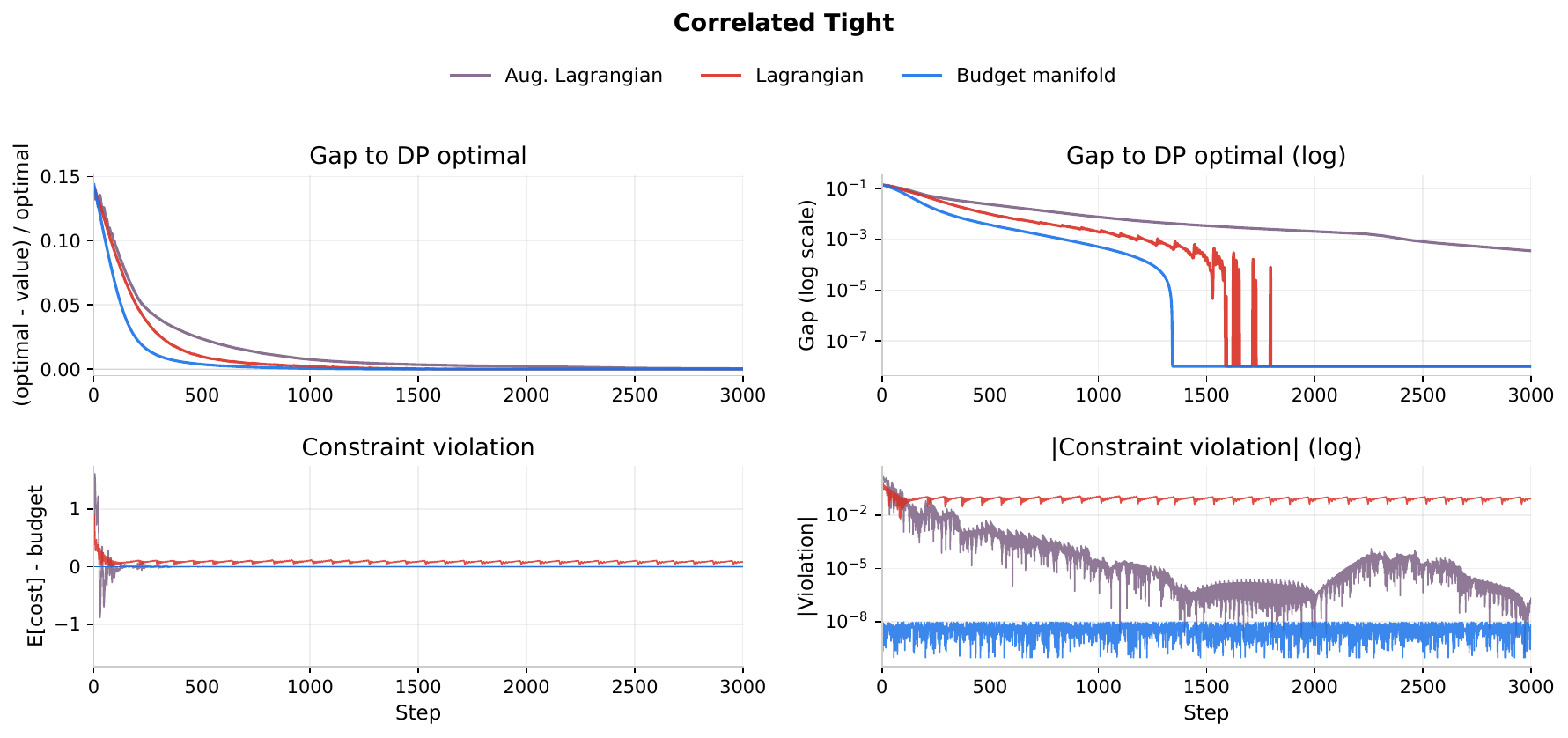}
    \caption{Correlated tight scenario ($N{=}50$, $K{=}8$, single instance).
    Strong cost-value correlation with budget only 15\% above minimum.
    Budget manifold converges faster (panels~a,b) and maintains $10^{-8}$ violation (panel~d) versus $10^{-2}$ for Lagrangian methods.}
    \label{fig:mckp_corr}
  \end{subfigure}
  \caption{MCKP convergence traces for adversarial and correlated tight scenarios.}
  \label{fig:mckp_adv_corr}
\end{figure}

\begin{figure}[H]
  \centering
  \begin{subfigure}[t]{\linewidth}
    \centering
    \includegraphics[width=\linewidth]{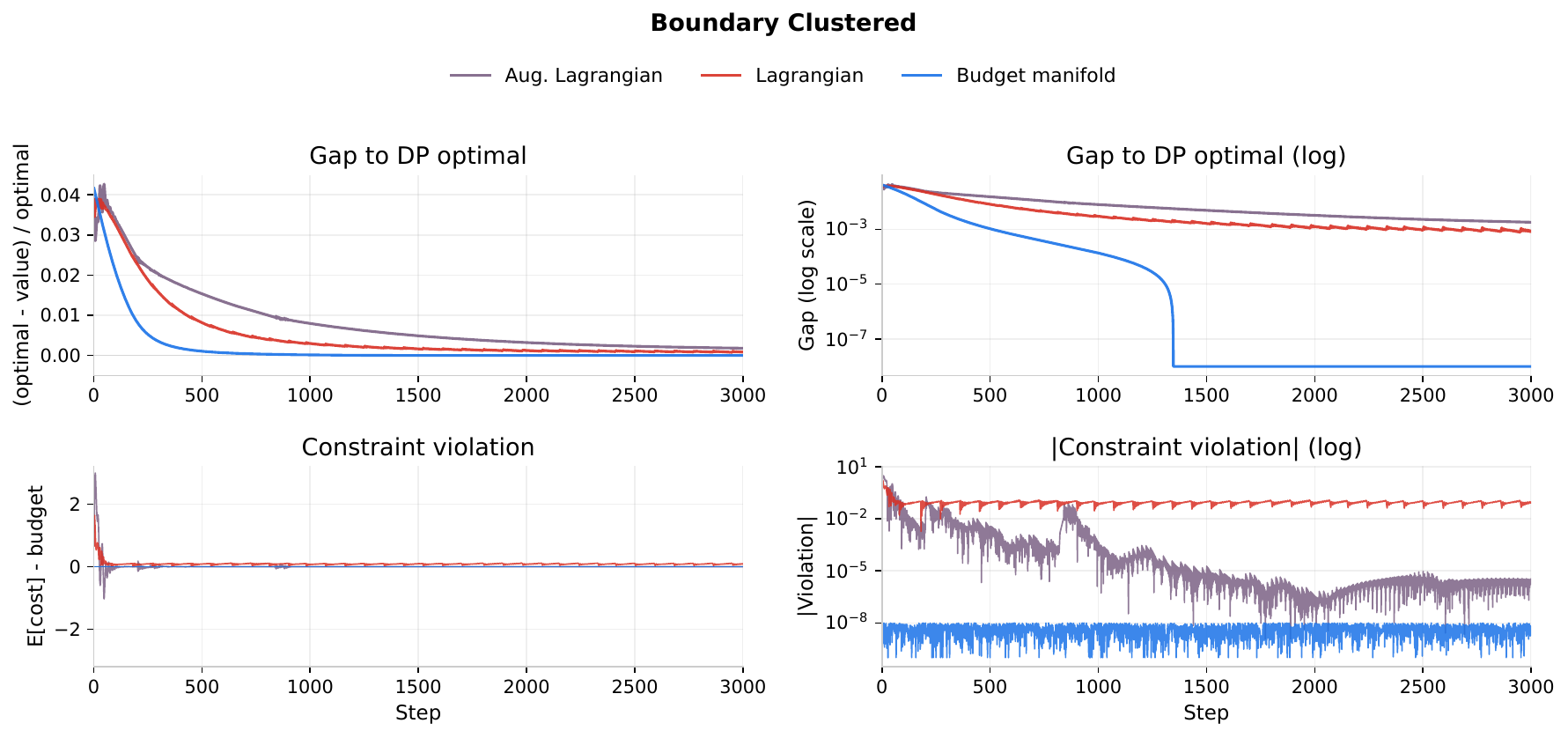}
    \caption{Boundary-clustered scenario ($N{=}50$, $K{=}8$, single instance).
    Many near-optimal solutions packed at the budget boundary.
    Log-scale convergence (panel~b) shows Budget manifold reaching $10^{-6}$ gap while Lagrangian methods plateau at $10^{-3}$.}
    \label{fig:mckp_boundary}
  \end{subfigure}

  \vspace{3em}

  \begin{subfigure}[t]{\linewidth}
    \centering
    \includegraphics[width=\linewidth]{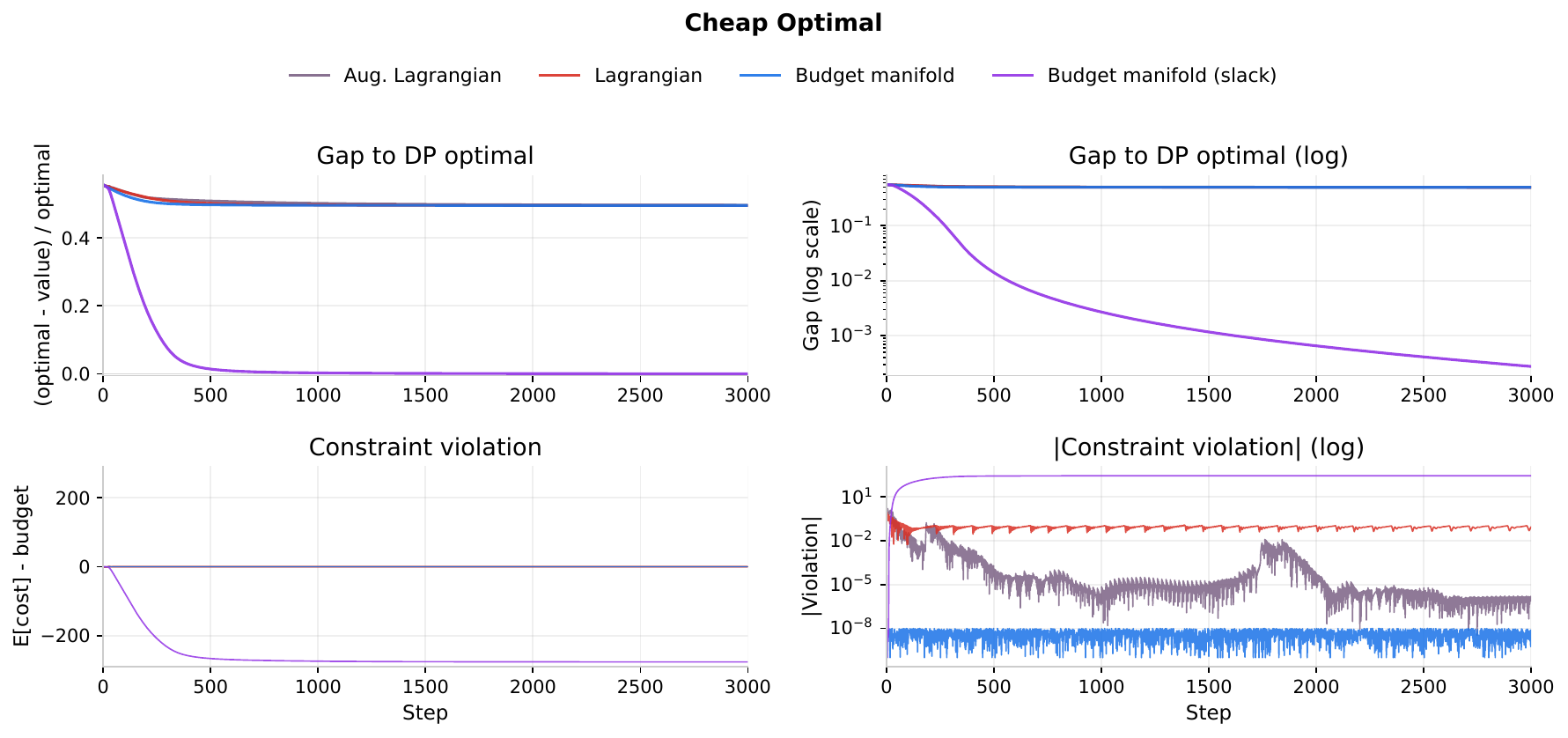}
    \caption{Cheap-optimal scenario ($N{=}50$, $K{=}8$, single instance) with slack variable.
    The DP optimum costs well under budget.
    Equality-constrained methods (Budget manifold (eq.), Lagrangian, Aug.~Lagrangian) all converge to $\sim$50\% of optimal (panel~a).
    Budget manifold with slack finds the true optimum, with the slack variable absorbing the unused budget (panel~c: $C(\alpha) - B \to -276$).}
    \label{fig:mckp_cheap}
  \end{subfigure}
  \caption{MCKP convergence traces for boundary-clustered and cheap-optimal scenarios.}
  \label{fig:mckp_bound_cheap}
\end{figure}

\end{document}